\documentclass[10pt,twocolumn,letterpaper]{article}
\PassOptionsToPackage{table}{xcolor}
\usepackage[pagenumbers]{cvpr}
\usepackage{amstext,amsfonts,bm,mathtools,amsthm}
\usepackage{array,colortbl,microtype,placeins,needspace}
\definecolor{mygray}{rgb}{0.90,0.9,1.00}
\definecolor{mdgshade}{RGB}{241,246,252}
\definecolor{ccdshade}{RGB}{241,249,242}
\definecolor{allshiftshade}{RGB}{253,248,239}

\newtheoremstyle{compactplain}
  {0pt}{0pt}{\itshape}{}{\bfseries}{.}{5pt plus 1pt minus 1pt}{}
\theoremstyle{compactplain}
\newtheorem{proposition}{Proposition}[section]
\newtheorem{lemma}{Lemma}[section]
\newtheorem{corollary}{Corollary}[section]

\newtheorem{assumption}{Assumption}[section]

\definecolor{cvprblue}{rgb}{0.21,0.49,0.74}
\usepackage[pagebackref,breaklinks,colorlinks,allcolors=cvprblue]{hyperref}
\def\confName{CVPR}
\def\confYear{2027}

\title{Regularizing patch contributions across domain and category shifts}
\author{
Zhaorui Tan$^{1}$ \qquad Weimiao Yu$^{1}$ \qquad Xi Yang$^{2}$\\
$^{1}$Bioinformatics Institute (BII), A*STAR, Singapore\\
$^{2}$Xi'an Jiaotong-Liverpool University, China\\
{\tt\small tan\_zhaorui@a-star.edu.sg \quad yu\_weimiao@a-star.edu.sg}\\
{\tt\small Xi.Yang01@xjtlu.edu.cn}
}
\hypersetup{
  pdftitle={Regularizing patch contributions across domain and category shifts},
  pdfauthor={Zhaorui Tan, Weimiao Yu, Xi Yang}
}

\begin{document}

\maketitle

\begin{abstract}
{Visual classification under domain and category shifts requires identifying which image regions provide transferable evidence. Domain alignment alone can retain context that is stable across source domains yet redundant for prediction, while new categories may require different local distinctions. We propose PatchGen to learn image-conditioned soft patch weights and regularize their contributions to the task representation. Low-score suppression acts directly on the weights, while selected-confidence and class-conditional compactness objectives shape the aggregated features using available labels or host-generated pseudo-labels. An oracle predictive-subset analysis characterizes information preservation, complexity upper bounds, and sensitivity to selection error under explicit assumptions.}
Experiments on natural and histopathological image benchmarks across all three shift settings show that PatchGen improves average performance over matched-backbone baselines in most evaluated configurations, enhances generalization to unknown classes, and remains competitive with vision-language methods without text supervision.
\end{abstract}

\section{Introduction}

Generalization remains challenging in visual classification~\cite{gulrajani2020search,duan2025hyperbolic,liu2021isometric,vaze2022generalized,cha2022miro,abbe2024generalization,tan2024rethinking,tan2024interpret} because real-world distribution shifts often combine \emph{data shifts} (e.g., domain changes) and \emph{target shifts} (e.g., new or evolving categories), jointly termed \emph{all shifts}~\cite{tan2024interpret} as shown in Table~\ref{tab:task_challenges}. {Domain shifts can change which contextual cues remain reliable, while category shifts can change which local details distinguish classes.} Even with strong pretrained patch representations from a visual foundation model (VM), it remains unclear which regions provide useful evidence and how to aggregate it without overusing co-occurring context. 

Existing visual generalization methods handle distribution shifts by suppressing domain-specific spurious correlations, typically via domain-invariant representations learned through statistical alignment~\cite{lin2025constrained,li2018domain,hu2020domain,cha2022miro,tan2024rethinking}
or causal-inspired formulations~\cite{gong2025causal,wang2025exploring,liu2025caurdg}.
When a category co-occurs with the same context across source domains, that association can survive domain alignment. Such context can correlate with the label across domains and still add little information once the object itself is visible. In Fig.~\ref{fig:path_exp1}(a), sky and sea co-occur with birds. The histopathology examples in Fig.~\ref{fig:path_exp1}(b) illustrate a similar distinction between tumor-related evidence and co-occurring inflammation. Category discovery poses a similar challenge, where silhouettes may distinguish birds from other animals, but wing markings may distinguish novel, similarly shaped species.

{Our insight is that patch contributions should be shaped by the distinctions required by the task, while preserving useful predictive information. Confidence alone is insufficient to guide this process, because co-occurring context can also support confident training predictions. We formalize this view through an image-specific oracle predictive subset (Fig.~\ref{fig:framework}(a)). Under explicit sufficiency assumptions, this subset preserves Bayes risk and admits a tighter complexity upper bound for fixed-selector linear predictors when fewer patches suffice (Prop.~\ref{prop:oracle_generalization}). This motivates learning image-adaptive patch weights together with objectives that constrain their predictive use.}

\makeatletter
\newsavebox{\PGmotivationfigure}
\newsavebox{\PGmotivationcaption}
\newsavebox{\PGmotivationtable}
\newlength{\PGmotivationtableheight}
\begin{figure*}[!t]
\centering
\begin{lrbox}{\PGmotivationfigure}
\begin{minipage}[b]{0.36\textwidth}
\def\@captype{figure}
{\fontsize{7}{7.6}\selectfont\centering
\begin{minipage}[t]{0.62\linewidth}
\vspace{0pt}\centering
\textbf{(a) Bird and tumor examples}\par
\setlength{\tabcolsep}{.4pt}\renewcommand{\arraystretch}{0.86}
\begin{tabular}{@{}lccc@{}}
 & {\fontsize{6}{6.5}\selectfont Natural} & {\fontsize{6}{6.5}\selectfont\shortstack{Pathology\\1}} & {\fontsize{6}{6.5}\selectfont\shortstack{Pathology\\2}}\\
\raisebox{10.75pt}{Original} & \includegraphics[width=25.5pt,trim=51bp 118bp 369bp 241bp,clip]{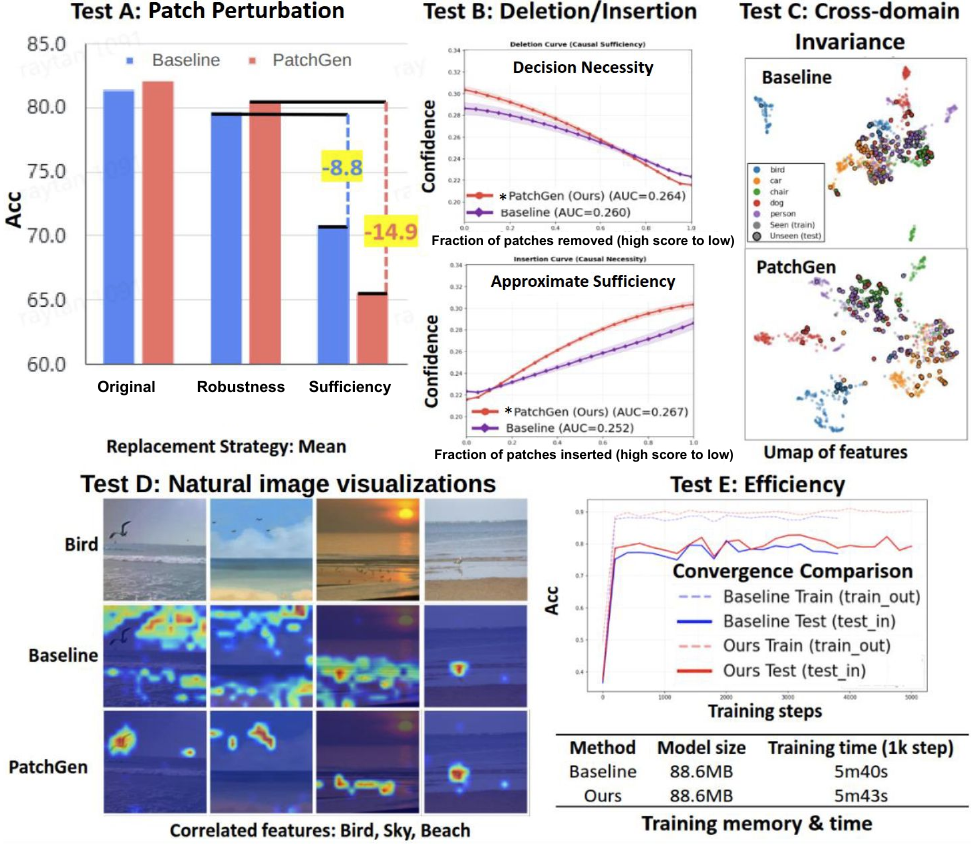} & \includegraphics[width=25.5pt]{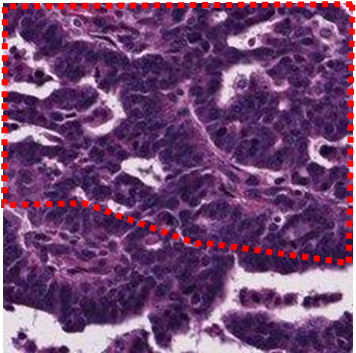} & \includegraphics[width=25.5pt]{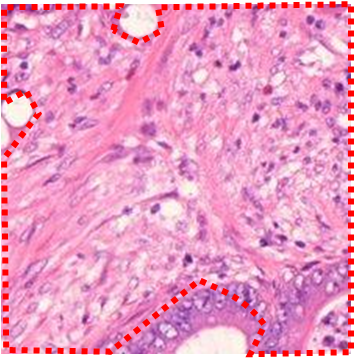}\\[-1pt]
\raisebox{6.75pt}{\shortstack{VM\\Baseline}} & \includegraphics[width=25.5pt,trim=51bp 68bp 369bp 291bp,clip]{imgs/res_vis/rebuttal_fig_refined_final.pdf} & \includegraphics[width=25.5pt]{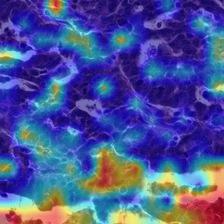} & \includegraphics[width=25.5pt]{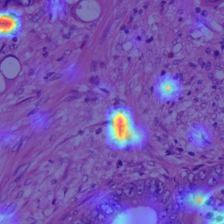}\\[-1pt]
\raisebox{10.75pt}{PatchGen} & \includegraphics[width=25.5pt,trim=51bp 18bp 369bp 341bp,clip]{imgs/res_vis/rebuttal_fig_refined_final.pdf} & \includegraphics[width=25.5pt]{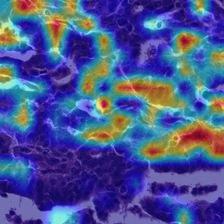} & \includegraphics[width=25.5pt]{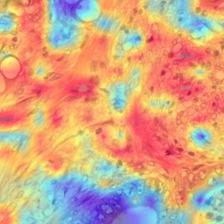}
\end{tabular}
\end{minipage}\hfill
\begin{minipage}[t]{0.36\linewidth}
\vspace{0pt}\centering
\textbf{(b) Possible relation}\par
{\fontsize{7}{7.6}\selectfont
(1) Tumor cells $\to$ (2) Disordered tissue architecture $\to$ (3) Inflammation.\par}
\setlength{\tabcolsep}{.5pt}
\begin{tabular}{@{}cc@{}}
\includegraphics[width=29pt,trim=0bp 21bp 443bp 0bp,clip]{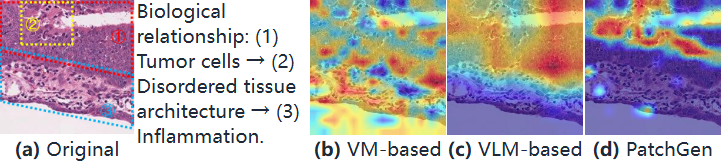} & \includegraphics[width=29pt,trim=233bp 21bp 209bp 0bp,clip]{imgs/res_vis/path_exp1_vis_refine.png}\\[-1pt]
{\fontsize{5.1}{5.6}\selectfont \mbox{(i) Original}} & {\fontsize{5.1}{5.6}\selectfont \mbox{(ii) VM-based}}\\[1pt]
\includegraphics[width=29pt,trim=335bp 21bp 106bp 0bp,clip]{imgs/res_vis/path_exp1_vis_refine.png} & \includegraphics[width=29pt,trim=438bp 21bp 0bp 0bp,clip]{imgs/res_vis/path_exp1_vis_refine.png}\\[-1pt]
{\fontsize{5.1}{5.6}\selectfont \mbox{(iii) VLM-based}} & {\fontsize{5.1}{5.6}\selectfont \mbox{(iv) PatchGen}}
\end{tabular}\par
\end{minipage}\par}
\setlength{\abovecaptionskip}{2pt}
\caption{\textbf{Visual motivation.}
(a) Original images (bird/tumor image-level labels), baseline and PatchGen response maps. Red outlines mark tumor regions on the pathology originals.
(b) Proposed tumor--disordered tissue architecture--inflammation relationship with VM-, VLM-, and PatchGen response maps.}
\label{fig:path_exp1}
\par\vspace{0pt}\end{minipage}
\end{lrbox}
\noindent\usebox{\PGmotivationfigure}\hfill
\begin{minipage}[b]{0.62\textwidth}
\def\@captype{table}
\begin{lrbox}{\PGmotivationcaption}\begin{minipage}{\linewidth}
\end{minipage}\end{lrbox}
\caption{\textbf{Three settings}: protocols, challenges, training signals, and conditional bounds.}
\label{tab:task_challenges}
\vspace{-9pt}
\begin{lrbox}{\PGmotivationtable}\begin{minipage}{\linewidth}
\fontsize{7.5}{8.2}\selectfont
\setlength{\tabcolsep}{2pt}
\renewcommand{\arraystretch}{1.0}
\begin{tabular}{@{}>{\raggedright\arraybackslash}p{40pt}>{\raggedright\arraybackslash}p{82pt}>{\raggedright\arraybackslash}p{68pt}>{\raggedright\arraybackslash}p{\dimexpr\linewidth-190pt-6\tabcolsep\relax}@{}}
\toprule
\textbf{Task / shift} & \textbf{Protocol} & \textbf{Challenge} & \textbf{Our training signal / bound} \\
\midrule
\rowcolor{mdgshade}\textbf{mDG} (data shift)~\cite{lin2025constrained,li2018domain,hu2020domain,cha2022miro,tan2024rethinking} & Labeled multi-source training; unseen-domain testing without target-domain training data. & Preserve class semantics while limiting the domain risk gap $\Delta_R$. & Source labels; $\Delta_R\le LBP(\epsilon_d+\epsilon_{d'})$. Smaller mask error tightens this upper bound (Prop.~\ref{prop:data_shift}). \\
\addlinespace[1pt]
\rowcolor{ccdshade}\textbf{CCD} (target shift)~\cite{cendra2024promptccd} & Start with a labeled session; then learn from sequential unlabeled streams. & Discover novel categories and retain known ones with class separation $M>0$. & Initial labels and host pseudo-labels; $M\ge\gamma-2(r+\eta)$. Smaller representation error raises this lower bound (Prop.~\ref{prop:target_shift}). \\
\addlinespace[1pt]
\rowcolor{allshiftshade}\textbf{mDG+GCD} (all shifts)~\cite{tan2024interpret} & Train on seen domains with labeled known classes and unknown-only unlabeled samples; test on unseen domains without target training data. & Recognize known classes and cluster unknown classes with small $\Delta_R$ and $M>0$. & Labels and host pseudo-labels; $\Delta_R\le LBP(\epsilon_d+\epsilon_{d'})$ and $M\ge\gamma-2(r+\eta)$ hold componentwise (Cor.~\ref{cor:all_shift}). \\
\midrule
\multicolumn{4}{@{}p{\linewidth}@{}}{\textbf{Notation:} $d,d'$ denote domains; $R_d(h,\phi)$ is expected loss for classifier $h$ and selector $\phi$; $\Delta_R=|R_d(h,\phi)-R_{d'}(h,\phi)|$; $\epsilon_d=\mathbb E_d[\|\mathbf m_\phi-\mathbf m_C^\star\|_1/P]$ is normalized error between learned and oracle masks; $L$ and $B$ bound loss Lipschitzness and patch-feature norms; $P$ is the patch count; $M$ is minimum learned cross-class sample distance; $\gamma$, $r$, and $\eta$ denote oracle center separation, within-class radius, and uniform learned-to-oracle representation error. Bounds require their stated assumptions.\par\vspace{1pt}\textbf{Shared mechanism:} pretrained patch features $\to$ image-conditioned soft mask $\to$ selected representation $\to$ task prediction.} \\
\bottomrule
\end{tabular}
\end{minipage}\end{lrbox}
\setlength{\PGmotivationtableheight}{\dimexpr
  \ht\PGmotivationfigure+\dp\PGmotivationfigure-
  \ht\PGmotivationcaption-\dp\PGmotivationcaption\relax}
\usebox{\PGmotivationcaption}\par\nointerlineskip
\centering\resizebox*{!}{\PGmotivationtableheight}{\usebox{\PGmotivationtable}}\par

\par\vspace{0pt}\end{minipage}
\vspace{-7pt}
\end{figure*}

We propose \textbf{PatchGen}, which learns a soft proxy by deriving image-specific patch weights from cross-patch interactions and using them to aggregate contextualized tokens (Fig.~\ref{fig:framework}(b)).
Low-score suppression directly regularizes these weights, while confidence and class-conditional compactness losses shape them via the aggregated representation\footnote{These objectives use ground-truth labels where available and host-generated pseudo-labels otherwise.}. {Unlike AF's adaptive token pruning for category discovery~\cite{xu2025distracted}, PatchGen regularizes soft patch contributions in the task representation under domain and category shifts.}

PatchGen improves generalization across 
multiple tasks (see Table~\ref{tab:task_challenges}) and their extensions, with gains over matched baselines on natural and histopathological images.
Notably, it achieves the highest average accuracy among the compared methods across five natural-image mDG benchmarks. {PatchGen also improves mean novel-class accuracy when applied to PromptCCD~\cite{cendra2024promptccd}, while perturbation diagnostics support the decision relevance of its selected regions.}
\par

{
Our main contributions are:
1)~We develop PatchGen, combining image-conditioned soft patch weighting with direct weight suppression and task-conditioned representation regularization for visual generalization.
2)~We show that module design and regularization make distinct contributions to performance, and provide cross-shift and perturbation evidence supporting the selection and use of transferable visual evidence.
3)~We characterize information preservation, complexity upper bounds, and sensitivity to selection error for oracle predictive subsets under explicit assumptions.}

\section{Related work}

\textbf{Generalization across shifts.}
mDG learns from source domains to generalize to unseen domains without target-domain training data, using class-conditional alignment~\cite{li2018deep,li2018domain,hu2020domain} or pretrained-representation regularization~\cite{cha2022miro,tan2024rethinking}. CCD discovers categories sequentially while limiting forgetting~\cite{cendra2024promptccd}; mDG+GCD combines unseen-domain generalization with known-class recognition and unknown-class clustering~\cite{tan2024interpret}.
\textbf{Causal methods and patch selection.}
Causal DG methods primarily address domain-level mechanisms~\cite{liu2025caurdg,wang2025exploring,gong2025causal}. Token pruning targets efficiency~\cite{rao2021dynamicvit}, while attention visualization supports post-hoc interpretation~\cite{abnar2020quantifying}. Sufficient input subsets preserve a fixed predictor's decision~\cite{carter2019sis}, differing from our conditional-sufficiency oracle. Attention-based approaches also address contextual confounding and background shortcuts~\cite{wang2021caam,shi2026registers}. 
{However, existing selection objectives do not by themselves explain why selected evidence should generalize across shifts. We connect intra-image predictive sufficiency to generalization bounds under explicit assumptions, and develop PatchGen to learn a soft predictive-subset mask.}
Appendix~\ref{app:detailed_related_work} provides the full discussion.

\section{Theoretical motivation}
\label{sec:theory}

\textbf{Definitions and structural assumptions.}
Complete proofs are provided in Appendix~\ref{app:proofs}.
Let $X\in\mathcal X$, $Y\in\mathcal Y$, and $D\in\mathcal D$
denote an input image, its class label, and its domain.
A backbone $f_\theta$ produces patch representations $Z(X)=[z_1(X),\ldots,z_P(X)]^\top\in\mathbb R^{P\times d_z}$, where $z_p(X)\in\mathbb R^{d_z}$ is the representation of patch $p\in[P]$.
{For a mask $\mathbf m\in[0,1]^P$, let define
$
\Phi_{\mathbf m}(X)=\operatorname{vec}(m_1z_1(X),\ldots,m_Pz_P(X))$.
An image-specific oracle subset $C^\star(X)\subseteq[P]$ has binary mask $\mathbf m_C^\star(X)$ and complement $R^\star(X)$. Write $\Phi_C=\Phi_{\mathbf m_C^\star}$, $\Phi_R=\Phi_{\mathbf 1_P-\mathbf m_C^\star}$, and $\Phi_{\mathrm{all}}=\Phi_{\mathbf 1_P}$. Let $\mathbb P_d$ and $\mathbb E_d$ denote probability and expectation in domain $d$.}

\begin{assumption}[Representation-measurable oracle partition]
\label{ass:oracle_measurability}
The oracle mask $\mathbf m_C^\star(X)$ is a deterministic function of
$\Phi_{\mathrm{all}}(X)$. Note: The mask introduces no side information unavailable in the full representation.
\end{assumption}

\begin{assumption}[Stable predictive sufficiency]
\label{ass:stable_sufficiency}
$\forall d\in\mathcal D$,
$
\mathbb P_d
\left(
Y\mid\Phi_{\mathrm{all}}
\right)
=
\mathbb P_d
\left(
Y\mid\Phi_C
\right)
=
\mathbb P
\left(
Y\mid\Phi_C
\right).
$
\end{assumption}

\begin{assumption}[Bounded oracle-subset size]
\label{ass:compact_support}
There exist constants $s\le P$ and $B>0$ such that
$
|C^\star(X)|\le s,
\qquad
\|z_p(X)\|_2\le B
$
for every image $X$ and $p\in[P]$.
\end{assumption}

{The oracle mask adds no information beyond the full representation. Predictive sufficiency requires $I_d(Y;\Phi_R\mid\Phi_C)=0$ but can allow $I_d(Y;\Phi_R)>0$. Thus, context may correlate with labels across domains yet be redundant given the retained evidence. The oracle subset may be non-unique and can include all patches when evidence is distributed.}

\subsection{Information preservation and complexity}
{For a fixed mask function $\mathbf m(X)$ and $\Lambda>0$, define
\begin{align}
\mathcal H_{\mathbf m}
=\{X\mapsto\langle\mathbf v,\Phi_{\mathbf m}(X)\rangle:
\|\mathbf v\|_2\le\Lambda\}.
\end{align}
Let $\mathcal H_C=\mathcal H_{\mathbf m_C^\star}$ and $\mathcal H_{\mathrm{all}}=\mathcal H_{\mathbf 1_P}$, and write $\widehat{\mathfrak R}_n$ for empirical Rademacher complexity.}

\begin{proposition}
[Oracle sufficiency and oracle-subset-size-dependent complexity]
\label{prop:oracle_generalization}
Under Assumptions~\ref{ass:stable_sufficiency}
and~\ref{ass:compact_support}, $\Phi_C$ and $\Phi_{\mathrm{all}}$ have equal unrestricted Bayes risk in every domain $d\in\mathcal D$. Their empirical Rademacher complexities satisfy
$\widehat{\mathfrak R}_n(\mathcal H_C) \le {\Lambda B\sqrt{s}}/{\sqrt n}, \widehat{\mathfrak R}_n(\mathcal H_{\mathrm{all}}) \le {\Lambda B\sqrt{P}}/{\sqrt n}.$
\end{proposition}

{Under standard feature and predictor norm bounds, the oracle bound is tighter when $s<P$, separating information preservation from the complexity cost of additional patches. The Bayes-risk result concerns unrestricted predictors; the complexity comparison concerns fixed-selector linear score classes. Corollary~\ref{cor:oracle_gen_bound} gives the corresponding generalization bound.}

\subsection{Error passed by the mask}

For a soft mask $\mathbf m_\phi(X)\in[0,1]^P$, fix a non-expansive map $G:\mathbb R^{Pd_z}\rightarrow\mathbb R^{d_g}$ and a classifier $h$ with loss $\ell$.
Define $\widetilde z_\phi^+(X)=
G\!\left(
\Phi_{\mathbf m_\phi(X)}(X)
\right)$,
$\widetilde z_C(X)
=
G\!\left(
\Phi_{\mathbf m_C^\star(X)}(X)
\right)$. For domain $d$, define the normalized mask error and the learned and oracle risks as
$
\epsilon_d(\phi)
=
\mathbb E_d
\left[
{
\|\mathbf m_\phi(X)-\mathbf m_C^\star(X)\|_1
}/{P}
\right]
$,
$R_d(h,C)=\mathbb E_d[\ell(h(\widetilde z_C(X)),Y)]$,
$R_d(h,\phi)
=
\mathbb E_d
\left[
\ell\!\left(
h(\widetilde z_\phi^+(X)),Y
\right)
\right]$.
We relate mask error to representation and risk deviations.

\begin{lemma}[Stability of masked concatenations]
\label{lem:mask_stability}
Assume $\|z_p(X)\|_2\le B$ for every patch, and let $G$ be a fixed
non-expansive map. Then, 
$\forall \mathbf m,\mathbf m'\in[0,1]^P$,
\begin{align}
\left\|
G\!\left(\Phi_{\mathbf m}(X)\right)
-
G\!\left(\Phi_{\mathbf m'}(X)\right)
\right\|_2 \nonumber \\
\le
B\|\mathbf m-\mathbf m'\|_2
\le
B\|\mathbf m-\mathbf m'\|_1.
\label{eq:mask_representation_stability}
\end{align}
\end{lemma}

\begin{lemma}[Risk deviation from the oracle predictive-subset mask]
\label{lem:oracle_risk_deviation}
Suppose $u\mapsto\ell(h(u),y)$ is $L$-Lipschitz.
Then
$|R_d(h,\phi)-R_d(h,C)| \le LBP\,\epsilon_d(\phi).$
\end{lemma}

{The bound links prediction stability to selection accuracy. We next examine its implications for the domain and category shifts in Table~\ref{tab:task_challenges}.}

\subsection{Implications across shift settings %
}
{The oracle analysis connects patch selection to two complementary goals: stable predictions across domains and preserved distinctions between categories.}

\textbf{Data shifts: complementary-context variation.}
\begin{proposition}[Robustness under complementary-context data shifts]
\label{prop:data_shift}
Consider two domains $d,d'\in\mathcal D$ with
$
\mathbb P_d(\Phi_C,Y)
=
\mathbb P_{d'}(\Phi_C,Y),
$
but potentially differing in their complementary-context distributions.
If the mapping
$
u\mapsto\ell(h(u),y)
$
is $L$-Lipschitz for all $\forall y\in\mathcal Y$, then
$
\left|
R_d(h,\phi)-R_{d'}(h,\phi)
\right|
\le
LBP
\left(
\epsilon_d(\phi)+\epsilon_{d'}(\phi)
\right).
$
\end{proposition}

Here, $L$ is the Lipschitz constant, $B$ bounds the patch-feature norm, 
and $P$ is the total number of patches.
 Thus, under this data shift, risk discrepancy is controlled by the mask-approximation error in both domains.

\textbf{Target shifts: preservation of class separation.}
Partition the label space as
$
\mathcal Y
=
\mathcal Y_{\mathrm{known}}
\mathbin{\dot\cup}
\mathcal Y_{\mathrm{unknown}}.
$
Class separation assumes oracle clusterability.

\begin{proposition}[Preservation of class separation under target shifts]
\label{prop:target_shift}
Suppose that there exist $r\ge0$, $\gamma>0$, and class centers
$
\mu_c\in\mathbb R^{d_g},
c\in
\mathcal Y_{\mathrm{known}}
\cup
\mathcal Y_{\mathrm{unknown}},
$
such that for any sample with $Y=c$,
$
\|\widetilde z_C(X)-\mu_c\|_2\le r,
$
and
$
\|\mu_c-\mu_{c'}\|_2\ge\gamma,
\forall c\ne c'.
$
If there exists $\eta\ge0$ such that
$
\|\widetilde z_\phi^+(X)-\widetilde z_C(X)\|_2
\le
\eta
$
for every sample, then learned distinct classes are separated by at least
$
\gamma-2(r+\eta).
$
In particular, the learned class-conditional representation sets remain
pairwise disjoint whenever
$
\gamma>2(r+\eta).
$
\end{proposition}

Under these assumptions, a uniform mask-error bound yields:

\begin{corollary}[Mask-error condition for separation]
\label{cor:mask_target_separation}
If
$\|\mathbf m_\phi(X)-\mathbf m_C^\star(X)\|_1\le q$
for every sample, then
$\eta\le Bq$ and the learned class-separation margin is at least
$\gamma-2(r+Bq).$
\end{corollary}

This uniform bound complements the expected mask-error bound used for context shifts.

\textbf{All shifts: componentwise implications.} Then we have:

\begin{corollary}[Componentwise implications under all shifts]
\label{cor:all_shift}
If the assumptions of
Props.~\ref{prop:data_shift} and~\ref{prop:target_shift}
hold simultaneously for the modeled data- and target-shift components, the learned predictor satisfies the complementary-context
risk-discrepancy bound in Prop.~\ref{prop:data_shift}, while learned representations of distinct classes remain separated by at least 
$
\gamma-2(r+\eta).
$
\end{corollary}

These conclusions hold componentwise; interactions between shift components require further theoretical analysis.

\section{Regularizing patch contributions}
\label{sec:method}

\begin{figure}[t]
    \centering
    \includegraphics[width=0.9\linewidth]
    {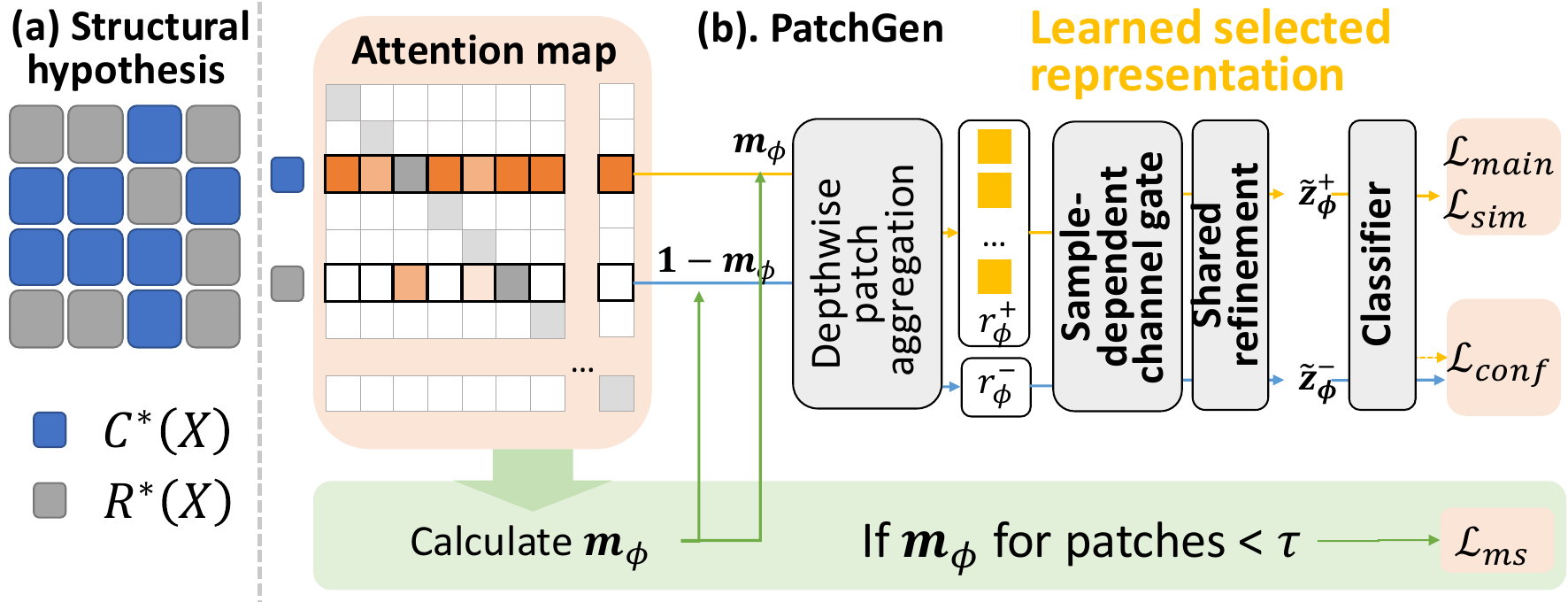}\par\nointerlineskip
    \setlength{\abovecaptionskip}{4pt}%
    \caption{
    \textbf{Overview of PatchGen.}
    {(a) Structural hypothesis.
    (b) PatchGen learns a task-driven soft-mask proxy for the oracle
    intra-image predictive subset.}
    }
    \label{fig:framework}
\vspace{-12pt}
\end{figure}

Fig.~\ref{fig:framework} summarizes PatchGen. An image-conditioned soft mask weights contextualized patch tokens to form the task representation. Training combines the host objective, direct suppression of weak mask responses, and confidence and class-conditional compactness objectives on the selected representation, implementing the predictive-subset perspective without patch-level annotations.

\textbf{Soft predictive-subset mask estimation.}
Under the structural hypothesis, the oracle predictive-subset
representation $\Phi_C(X)$ is sufficient for label prediction.
Because $\mathbf m_C^\star(X)$ is unavailable during training, we estimate
the learned soft predictive-subset mask $\mathbf m_\phi(X)$ from
cross-patch interaction scores.
Let $S_\phi^{(a)}(X)\in\mathbb R^{P\times P}$ denote the pre-softmax interaction-score matrix of attention head $a$, where $H$ is the number of heads and $d_k$ is the query/key dimension of each head:
$
S_{\phi,p\rightarrow q}^{(a)}(X)
=
{
\left\langle
q_{\phi,p}^{(a)}(X),
k_{\phi,q}^{(a)}(X)
\right\rangle
}/{
\sqrt{d_k}
}.
$
The patch-selection score is obtained by averaging the outgoing interaction scores over heads and target patches,
\begin{equation}
m_{\phi,p}(X)
=
\sigma\!\left(
{1}/{HP}
\sum\nolimits_{a=1}^{H}
\sum\nolimits_{q=1}^{P}
S_{\phi,p\rightarrow q}^{(a)}(X)
\right),
\label{eq:predictive_patch_mask}
\end{equation}
where $\sigma(\cdot)$ denotes the sigmoid function.
The resulting soft mask is $\mathbf m_\phi(X)= (m_{\phi,1}(X),\ldots,m_{\phi,P}(X))\in(0,1)^P$.
Query and key projections are initialized to keep interaction scores moderate and non-saturated, allowing selectivity to develop during training. Unlike post-hoc attention visualization~\cite{abnar2020quantifying}, attention parameters and mask $\mathbf m_\phi(X)$ are jointly optimized with the objectives below.
Table~\ref{tab:ablation} compares standard attention and attention with only $\mathcal L_{ms}$ (Tasks~2.3--2.4) against the complete PatchGen objective (Tasks~3.3--3.5).

Softmax-normalized interaction scores are used to construct
attention-contextualized patch representations
$U_\phi(X)=[u_{\phi,1}(X),\ldots,u_{\phi,P}(X)]^\top$.
For each headly ,
\begin{equation}
A_{\phi,p\rightarrow q}^{(a)}(X)
=
\operatorname{softmax}_{q}
\left(
{S_{\phi,p\rightarrow q}^{(a)}(X)}/
{\exp(t_\phi)}
\right),
\end{equation}
where  $t_\phi$ is a
learned log-temperature parameter and the multi-head attention output defines $U_\phi(X)$.
A learnable channel-wise patch aggregator
$\mathcal A_\rho$ constructs the selected and complementary aggregates,
\begin{equation}
r_\phi^+
=
\mathcal A_\rho
\left(
U_\phi,\mathbf m_\phi
\right), \quad
r_\phi^-
=
\mathcal A_\rho
\left(
U_\phi,\mathbf 1_P-\mathbf m_\phi
\right).
\label{eq:selected_complementary_features}
\end{equation}

The aggregator uses a depthwise full-width one-dimensional convolution over patches, with kernel size equal to the patch count $P$.
It thus learns a channel-specific global aggregation of masked patch features. For feature channel
$j\in\{1,\ldots,d_z\}$,
$[\mathcal A_\rho(U,\mathbf m)]_j
=
\sum\nolimits_{p=1}^{P}
\rho_{j,p}\,m_p\,u_{p,j}.$
Weights $\rho\in\mathbb R^{d_z\times P}$ are learned and shared across branches, without bias or extra normalization.

A sample-dependent channel gate $\mathbf w_\psi(X) = g_\psi(r_\phi^+) \in(0,1)^{d_z}$ and shared refinement map $F_\omega$ yielded the final selected and complementary representations,
\begin{equation}
    \widetilde z_\phi^+=F_\omega\left(\mathbf w_\psi(X)\odot r_\phi^+\right), \quad \widetilde z_\phi^-=F_\omega\left(\mathbf w_\psi(X)\odot r_\phi^-\right).
\end{equation}
The classifier uses $\widetilde z_\phi^+$. Since attention precedes masking, selected features may contain other-patch information. Theory concerns idealized masked representations; implementation aggregates contextualized features. %

\textbf{Low-score mask suppression.}
For minibatch $\mathcal B$, let the weak mask entries be,
\begin{equation}
  \mathcal W_\tau(\mathcal B)=
\{
(i,p):
i\in\mathcal B,\;
p\in[P],\;
m_{\phi,i,p}<\tau
\}.  
\end{equation}
The mask-suppression loss is,
\begin{equation}
    \mathcal L_{ms} = {1}/{|\mathcal W_\tau(\mathcal B)|} \cdot \sum\nolimits_{(i,p)\in\mathcal W_\tau(\mathcal B)} m_{\phi,i,p} \;,
\end{equation}
when $|\mathcal W_\tau(\mathcal B)|>0$ , otherwise $\mathcal L_{ms} =0$, with $\tau=0.25$ throughout.
It pushes weak mask responses toward zero and sharpens the
separation between weakly and strongly weighted patches.
It leaves selected-set size image-adaptive: when all scores exceed $\tau$, $\mathcal W_\tau=\varnothing$ and $\mathcal L_{ms}=0$, allowing the task objective to retain broadly distributed evidence.

\begin{table}[!t]
\setlength{\abovecaptionskip}{3pt}

\centering
\caption{\textbf{MDG results for natural images}: Comparison between the proposed and previous MDG methods, including VLM-based, VM-based methods, and causal-based methods using VMs.
Dataset abbreviations (also used in Tables~\ref{tab:mDG_results_causal}, \ref{tab:mDG_results_full_finetune}, and~\ref{tab:ablation}): \textbf{P}: PACS, \textbf{V}: VLCS, \textbf{O}: OfficeHome, \textbf{T}: TerraIncognita, \textbf{D}: DomainNet.
{\textbf{Note:} For DINOv2 ViT-L/14, LP and LP+LN match after rounding; their unrounded averages differ by less than 0.1 percentage points.}
}
\label{tab:mDG_results}
{\fontsize{7}{8}\selectfont
\setlength{\tabcolsep}{1.8pt}
\renewcommand{\arraystretch}{0.94}
\setlength{\aboverulesep}{0.4pt}
\setlength{\belowrulesep}{0.5pt}
\setlength{\cmidrulesep}{0.4pt}
\resizebox{\linewidth}{!}{%
\begin{tabular}{@{}lcccccc@{}}
\toprule
Method / Test domain & \textbf{P} & \textbf{V} & \textbf{O} & \textbf{T} & \textbf{D} & \textbf{Avg.} \\
\midrule
\multicolumn{7}{@{}l@{}}{\textbf{VLM-based methods}; Pretrain: CLIP ViT-B/16} \\
\cmidrule(lr){1-7}
SWAD~\cite{cha2021swad} (NIPS21) & 91.3 & 79.4 & 76.9 & 45.4 & 51.7 & 68.9 \\
CLIP~\cite{radford2021learning}  (ICML21) & 96.2 & 81.7 & 82.0 & 33.4 & 57.5 & 70.2 \\
SMA~\cite{arpit2022ensemble}  (NIPS22) & 92.1 & 79.7 & 78.1 & 48.3 & 55.9 & 70.8 \\
DUPRG~\cite{niu2022domain}  (ICLR23) & 97.1 & 83.9 & 83.6 & 42.0 & 59.6 & 73.2 \\
CoOp~\cite{zhou2022learning}  (IJCV22) & 96.2 & 77.6 & 83.9 & 48.8 & 59.8 & 73.3 \\
MIRO~\cite{cha2022miro}  (ECCV22) & 95.6 & 82.2 & 82.5 & 54.3 & 54.0 & 73.7 \\
DPL~\cite{zhang2023domain}  (TAI23) & 97.3 & 84.3 & 84.2 & 52.6 & 56.7 & 75.0 \\
CLIPOOD~\cite{shu2023clipood}  (ICML23) & 97.3 & 85.0 & 87.0 & 60.4 & 63.5 & 78.6 \\
Promptstyler~\cite{cho2023promptstyler}  (ICCV23) & 97.2 & 82.9 & 83.6 & - & 59.4 & - \\
KAdaptation~\cite{lee2025domain}  (WACV25) & 97.5 & 83.0 & 90.3 & 51.9 & 62.7 & 77.1 \\
GESTUR~\cite{lew2023gestur}  (ICCV23) & 96.0 & 82.8 & 84.2 & 55.7 & 58.9 & 75.5 \\
DPR~\cite{cheng2024disentangled}  (CVPR24) & 97.5 & 86.4 & 86.1 & 57.1 & 62.1 & 77.8 \\
CLIPCEIL++~\cite{yu2024clipceil}  (NeurIPS24) & 97.2 & 85.2 & 87.7 & 62.0 & 63.6 & 79.1 \\
\midrule
\multicolumn{7}{@{}l@{}}{\textbf{VM-based methods}; Pretrain: ResNet50} \\
\cmidrule(lr){1-7}
MMD~\cite{li2018domainMMD} & 84.7±0.5 & 77.5±0.9 & 66.3±0.1 & 42.2±1.6 & 23.4±9.5 & 58.8 \\
Mixstyle~\cite{zhou2021domain} & 85.2±0.3 & 77.9±0.5 & 60.4±0.3 & 44.0±0.7 & 34.0±0.1 & 60.3 \\
GroupDRO~\cite{sagawa2019distributionally} & 84.4±0.8 & 76.7±0.6 & 66.0±0.7 & 43.2±1.1 & 33.3±0.2 & 60.7 \\
IRM~\cite{arjovsky2019invariant} & 83.5±0.8 & 78.5±0.5 & 64.3±2.2 & 47.6±0.8 & 33.9±2.8 & 61.6 \\
ARM~\cite{zhang2021adaptive} & 85.1±0.4 & 77.6±0.3 & 64.8±0.3 & 45.5±0.3 & 35.5±0.2 & 61.7 \\
VREx~\cite{krueger2021out} & 84.9±0.6 & 78.3±0.2 & 66.4±0.6 & 46.4±0.6 & 33.6±2.9 & 61.9 \\
CDANN~\cite{li2018deep} & 82.6±0.9 & 77.5±0.1 & 65.8±1.3 & 45.8±1.6 & 38.3±0.3 & 62.0 \\
DANN~\cite{ganin2016domain} & 83.6±0.4 & 78.6±0.4 & 65.9±0.6 & 46.7±0.5 & 38.3±0.1 & 62.6 \\
RSC~\cite{huang2020self} & 85.2±0.9 & 77.1±0.5 & 65.5±0.9 & 46.6±1.0 & 38.9±0.5 & 62.7 \\
MTL~\cite{blanchard2021domain} & 84.6±0.5 & 77.2±0.4 & 66.4±0.5 & 45.6±1.2 & 40.6±0.1 & 62.9 \\
MLDG~\cite{li2018learning} & 84.9±1.0 & 77.2±0.4 & 66.8±0.6 & 47.7±0.9 & 41.2±0.1 & 63.6 \\
Fish~\cite{shi2021gradient} & 85.5±0.3 & 77.8±0.3 & 68.6±0.4 & 45.1±1.3 & 42.7±0.2 & 63.9 \\
ERM~\cite{vapnik1999overview} & 84.2±0.1 & 77.3±0.1 & 67.6±0.2 & 47.8±0.6 & 44.0±0.1 & 64.2 \\
SagNet~\cite{nam2021reducing} & {86.3}±0.2 & 77.8±0.5 & 68.1±0.1 & 48.6±1.0 & 40.3±0.1 & 64.2 \\
SelfReg~\cite{kim2021selfreg} & 85.6±0.4 & 77.8±0.9 & 67.9±0.7 & 47.0±0.3 & 42.8±0.0 & 64.2 \\
CORAL~\cite{sun2016deep} & 86.2±0.3 & 78.8±0.6 & 68.7±0.3 & 47.6±1.0 & 41.5±0.1 & 64.5 \\
mDSDI~\cite{bui2021exploiting} & 86.2±0.2 & 79.0±0.3 & 69.2±0.4 & 48.1±1.4 & 42.8±0.1 & 65.1 \\
\midrule
\multicolumn{7}{@{}l@{}}{\textbf{VM-based methods}; Pretrain: RegNetY-16GF} \\
\cmidrule(lr){1-7}
MIRO+SWAD~\cite{cha2022miro} (ECCV22) & {97.4}±0.2 & 79.9±0.6 & 80.4±0.2 & 58.9±1.3 & 53.8±0.1 & 74.1 \\
{GMDG+SWAD}~\cite{tan2024rethinking} (CVPR24) & 97.3±0.1 & {82.4}±0.6 & {80.8}±0.6 & {60.7}±1.8 & {54.6}±0.1 & {75.1} \\
{L-Reg+SWAD}~\cite{tan2024interpret} (NeurIPS24) & {97.4}±0.2 & {82.4}±0.0 & {80.9}±0.5 & {62.9}±0.9 & {55.3}±0.0 & {75.8} \\
\midrule
\multicolumn{7}{@{}l@{}}{\textbf{Results of baselines and ours}; Pretrain: DINOv2 ViT-B/14} \\
\cmidrule(lr){1-7}
LP* & 93.0 & 77.0 & 84.2 & 56.7 & 53.5 & 72.9 \\
LP+LN finetune (Baseline)* & 96.2 & 80.0 & 86.1 & 61.5 & 56.3 & 76.0 \\
\rowcolor{mygray}\textbf{PatchGen} & 96.4 & 82.4 & 83.7 & 65.8 & 56.0 & 76.9 \\
Baseline+SWAD* & 96.7 & 80.3 & 86.8 & 64.7 & 58.3 & 77.4 \\
\rowcolor{mygray}\textbf{PatchGen}+SWAD & 97.1 & 82.5 & 87.5 & 69.7 & 59.2 & 79.2 \\
\midrule
\multicolumn{7}{@{}l@{}}{\textbf{Results of baselines and ours}; Pretrain: DINOv2 ViT-L/14} \\
\cmidrule(lr){1-7}
LP* & 97.0 & 78.6 & 86.7 & 65.7 & 54.3 & 76.5 \\
LP+LN finetune (Baseline)* & 97.0 & 78.6 & 86.7 & 65.7 & 54.3 & 76.5 \\
Baseline+SWAD* & 97.5 & 81.0 & 89.7 & 67.4 & 58.1 & 78.7 \\
Baseline+GMDG* & 96.1 & 79.2 & 85.7 & 67.9 & 54.0 & 76.6 \\
\rowcolor{mygray}\textbf{PatchGen} & 97.1 & 83.6 & 88.0 & 66.7 & 59.2 & 78.9 \\
Baseline+GMDG+SWAD* & 96.8 & 81.1 & 89.5 & 68.5 & 57.7 & 78.7 \\
\rowcolor{mygray}\textbf{PatchGen}+SWAD & 97.6 & 83.4 & 90.6 & 70.2 & 62.7 & 80.9 \\
\midrule
\multicolumn{7}{@{}l@{}}{\textbf{Results of baselines and ours}; Pretrain: MAE ViT-B/16} \\
\cmidrule(lr){1-7}
LP* & 58.4 & 75.9 & 49.5 & 26.1 & 17.0 & 45.4 \\
LP+LN finetune (Baseline)* & 70.1 & 78.5 & 53.2 & 25.0 & 25.2 & 50.4 \\
\rowcolor{mygray}\textbf{PatchGen} & 72.5 & 77.5 & 50.7 & 32.0 & 25.6 & 51.7 \\
Baseline+SWAD* & 72.1 & 78.5 & 56.3 & 24.5 & 26.6 & 51.6 \\
\rowcolor{mygray}\textbf{PatchGen}+SWAD & 73.9 & 79.0 & 55.1 & 27.3 & 27.2 & 52.5 \\
\bottomrule
\end{tabular}%
}
}
\vspace{-12pt}
\end{table}

\textbf{Selected-confidence regularization.}
Let $p_h(y\mid z)$ denote the predicted probability of class $y$. Targets are ground-truth labels for labeled samples and the host discovery objective's non-differentiable pseudo-labels $\bar Y_i$ for unlabeled samples in CCD and mDG+GCD. Let $\bar{\mathcal B}$ contain samples with either target available. We use the bounded confidence objective
$\mathcal L_{\mathrm{conf}}=-1/{|\bar{\mathcal B}|}\sum\nolimits_{i\in\bar{\mathcal B}}p_h(\bar Y_i\mid\widetilde z_{\phi,i}^{+}).$
It encourages $\widetilde z_{\phi}^{+}$ to retain target evidence, balancing low-score suppression by $\mathcal L_{ms}$ with task-predictive feature retention. A detached complementary-confidence reference gives the same gradient; Appendix~\ref{app:method-details} provides the equivalent formulation.

\textbf{Class-conditional feature alignment.}
For minibatch $\mathcal B$, let $\bar Y_i$ be the ground-truth label $Y_i$ for labeled samples and the task-specific pseudo-label otherwise. Define the active labels $\bar{\mathcal Y}_{\mathcal B}=\{\bar Y_i:i\in\mathcal B\}$ and class-specific index sets $\mathcal I_c=\{i\in\mathcal B:\bar Y_i=c\}$. The class-conditional similarity loss is,
\begin{equation*}
    \mathcal L_{sim} =\tfrac{1}{|\bar{\mathcal Y}_{\mathcal B}|} \sum\nolimits_{c\in\bar{\mathcal Y}_{\mathcal B}} \Bigl(1-\tfrac{1}{|\mathcal I_c|^2} \sum\nolimits_{i,j\in\mathcal I_c} \cos\!\bigl(\widetilde z_{\phi,i}^+,\widetilde z_{\phi,j}^+\bigr)\Bigr).
\end{equation*}
{Self-pairs are included, so singleton classes add zero. In multi-domain tasks, minibatches cover all source domains, allowing cross-domain same-class alignment when such pairs occur.  Thus, $\mathcal L_{sim}$ promotes class-conditional compactness of selected features, including cross-domain alignment, without explicit class prototypes. Unlabeled samples use host pseudo-labels from PromptCCD/L-Reg for CCD/mDG+GCD, respectively.}

\textbf{Overall learning objective.}
Using $\widetilde z_\phi^+$ as the task representation, PatchGen minimizes, %
\begin{equation}
\mathcal L_{\mathrm{PatchGen}}
=
\mathcal L_{\mathrm{main}}
+
\lambda_{ms}\mathcal L_{ms}
+
\lambda_{sim}\mathcal L_{sim}
+
\mathcal L_{\mathrm{conf}}.
\label{eq:overall_objective}
\end{equation}
$\mathcal L_{\mathrm{main}}$ is the mDG, CCD, or mDG+GCD objective. $\mathcal L_{ms}$ suppresses weak mask responses, $\mathcal L_{sim}$ aligns within-class features, and $\mathcal L_{\mathrm{conf}}$ regularizes selected confidence with the same gradient as a detached complementary-confidence reference. Unlike standard attention, PatchGen learns a sample-dependent soft proxy for an oracle intra-image predictive subset, distinguishing it from context that may be domain-varying, domain-stable but label-correlated, or marginally uninformative, instead of relying on domain invariance.

\section{Experiments}

We evaluate PatchGen’s transfer across domain and category shifts, the contribution of regularization beyond model capacity, and the decision relevance of its patch scores.

Unless stated otherwise, we use seed $1$ and shared hyperparameters across datasets (sensitivity: Fig.~\ref{fig:analysis}). Only backbone LayerNorm parameters are fine-tuned to preserve pretrained representations with parameter-efficient adaptation~\cite{de2023effectiveness,tan2025exploiting}. An asterisk marks official-code reimplementations or reevaluations; training code, configurations, and data-split scripts will be released upon publication.
For mDG/mDG+GCD, hyperparameters and checkpoints use source-only validation, excluding held-out-domain samples, labels, and statistics. CCD follows PromptCCD's current-session validation, excluding future-session test labels. PatchGen losses use detached PromptCCD/L-Reg pseudo-labels for CCD/mDG+GCD, respectively; alternative generators are out of scope.

\subsection{MDG experiments for natural images}
Table~\ref{tab:mDG_results} compare representative \textbf{VLM-based methods}:
SWAD~\cite{cha2021swad} with VLM, KAdaptation~\cite{lee2025domain},
GESTUR~\cite{lew2023gestur},
DPR~\cite{cheng2024disentangled},
CLIPCEIL++~\cite{yu2024clipceil}; \textbf{VM-based methods}: ERM~\cite{vapnik1999overview},
SagNet~\cite{nam2021reducing},
SelfReg~\cite{kim2021selfreg},
CORAL~\cite{sun2016deep},
mDSDI~\cite{bui2021exploiting},
{GMDG}~\cite{tan2024rethinking},
{L-Reg}~\cite{tan2024interpret};
\textbf{Causal-based methods}:
CI-DGA~\cite{gong2025causal},
SMIDG~\cite{wang2025exploring},
BFMix~\cite{zhang2024mix},
CauRDG~\cite{liu2025caurdg}.
\textbf{Experimental settings.} We follow DomainBed's leave-one-domain-out protocol~\cite{gulrajani2020search} on PACS~\cite{li2017deeper}, VLCS~\cite{fang2013unbiased}, OfficeHome~\cite{venkateswara2017deep}, TerraIncognita~\cite{beery2018recognition}, and DomainNet~\cite{peng2019moment}, using DINOv2~\cite{oquab2023DINOv2} (ViT-B/14, ViT-L/14) and MAE~\cite{he2022masked}. Matched controls share backbones and common hyperparameters: linear probing (LP), LP with LayerNorm tuning (LP+LN), and GMDG~\cite{tan2024rethinking} on DINOv2 ViT-L/14. SWAD is included where applicable. Source-only validation selects $\lambda_{ms},\lambda_{sim}\in\{0.01,0.001,0.0001\}$.

\begingroup
\setlength{\intextsep}{0pt}

\begin{figure}[!t]
\centering
\includegraphics[width=\linewidth]{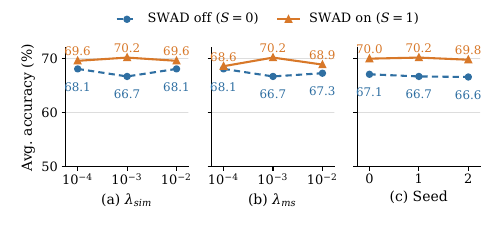}
\vspace{-25pt}
\caption{TerraIncognita mDG sensitivity and robustness (Avg. accuracy). $S$: SWAD off/on (0/1).
(a) Varying $\lambda_{sim}$ with $\lambda_{ms}=10^{-3}$ and seed 1;
(b) varying $\lambda_{ms}$ with $\lambda_{sim}=10^{-3}$ and seed 1;
(c) varying the seed with $\lambda_{ms}=\lambda_{sim}=10^{-3}$. Detailed results are in Appendix~\ref{app:detailed-results}, Table~\ref{tab:analysis}.}
\label{fig:analysis}
\vspace{-6pt}
\end{figure}

\begin{table}[!t]
\begingroup
\fontsize{7.5}{8.3}\selectfont
\setlength{\tabcolsep}{1.3pt}\renewcommand{\arraystretch}{0.94}
\setlength{\abovecaptionskip}{3pt}
\setlength{\aboverulesep}{0.22ex}\setlength{\belowrulesep}{0.22ex}

\centering
\caption{\textbf{MDG results for natural images}: Comparison between the proposed and selected causal-inspired DG methods using VMs.
}
\label{tab:mDG_results_causal}
\begin{tabular*}{\linewidth}{@{}lc@{\extracolsep{\fill}}ccccc|c@{}}
 \toprule
 Test domain & Backbone & \bf P & \bf V & \bf O & \bf T & \bf D& \bf Avg. \\ \midrule
CI-DGA~\cite{gong2025causal} (IJCV25) & DeiT-S & 88.3&80.6&72.6&-&-&- \\
SMIDG~\cite{wang2025exploring} (TPAMI25) & ResNet50 & 90.4 & 80.8 &72.1 & 52.8 & 45.5 & 68.3 \\
BFMix~\cite{zhang2024mix} (ICME25) & ViT-S/16 & 82.3 & 90.1 & 82.3 & 51.5 & 53.2 & 71.9 \\
CauRDG~\cite{liu2025caurdg} (ACMMM25)& ViT-B/16 & 89.8 &-&75.3&-&49.8 & - \\
\rowcolor{mygray}\textbf{PatchGen} & ViT-S/14 & 93.8 & 81.8 & 79.1 & 54.4 & 50.7 & 72.0 \\
\rowcolor{mygray}\textbf{PatchGen} + SWAD & ViT-S/14& 94.8&82.9&82.0&55.8&52.3&73.6 \\
\rowcolor{mygray}\textbf{PatchGen} & ViT-B/14 &93.9&82.6&84.0&56.1&55.3&74.4\\
\rowcolor{mygray}\textbf{PatchGen} + SWAD & ViT-B/14&96.6&82.7&85.3&62.3&59.5&77.3 \\

\bottomrule 
\end{tabular*}%
\par\endgroup
\vspace{-10pt}
\end{table}

\begin{table*}[!t]
\centering
\caption{\textbf{MDG results for histopathological images}: per-domain and average results of PatchGen and baselines across backbones. Entries are without/with SWAD~\cite{cha2021swad}; --: results not available.}
\label{tab:mdg_histopathology}
\fontsize{7}{7.8}\selectfont
\setlength{\tabcolsep}{1pt}
\renewcommand{\arraystretch}{0.96}
\setlength{\aboverulesep}{0.22ex}\setlength{\belowrulesep}{0.22ex}
\begin{minipage}[t]{0.49\linewidth}
\vspace{0pt}
\begin{tabular*}{\linewidth}{@{\extracolsep{\fill}}>{\raggedright\arraybackslash}p{88pt}ccccc@{}}
\multicolumn{6}{c}{\textbf{(a) HISTOPANTUM: two fully aligned classes}}\\[1pt]
\toprule
Method & Colon & Ovarian & Stomach & Uterus & Avg. \\
\midrule
\multicolumn{6}{@{}l}{\textit{Backbone: CLIP ViT-B/16}} \\
GESTUR* & 90.1/-- & 96.6/-- & 87.4/-- & 90.1/-- & 91.0/-- \\
LP* & 80.3/81.1 & 91.4/91.3 & 79.0/79.2 & 91.0/89.9 & 85.4/85.4 \\
LP + LN finetune (Baseline)* & 83.5/84.5 & 95.1/94.9 & 86.6/86.0 & 89.9/89.6 & 88.8/88.8 \\
\rowcolor{mygray}\textbf{PatchGen} & 86.9/87.2 & 95.1/95.6 & 86.2/85.1 & 85.6/90.7 & 88.5/89.6 \\
\midrule
\multicolumn{6}{@{}l}{\textit{Backbone: DINOv2 ViT-B/14}} \\
LP* & 87.3/88.4 & 93.8/94.6 & 84.1/85.0 & 90.8/87.2 & 89.0/88.8 \\
LP + LN finetune (Baseline)* & 91.4/90.7 & 96.7/96.7 & 88.8/88.5 & 88.5/89.9 & 91.4/91.4 \\
Baseline + MIRO* & 89.1/90.8 & 96.2/96.8 & 87.6/87.8 & 91.8/88.8 & 91.2/91.0 \\
Baseline + GMDG* & 88.9/90.3 & 96.2/96.5 & 88.0/87.2 & 90.1/85.5 & 90.8/89.9 \\
\rowcolor{mygray}\textbf{PatchGen} & 91.0/92.1 & 96.2/97.0 & 89.0/88.3 & 94.9/93.9 & 92.8/92.8 \\
\midrule
\multicolumn{6}{@{}l}{\textit{Backbone: DINOv2 ViT-L/14}} \\
LP* & 90.3/88.9 & 92.4/94.5 & 84.5/84.2 & 93.1/90.9 & 90.1/89.6 \\
LP + LN finetune (Baseline)* & 92.8/93.4 & 97.1/96.7 & 88.5/88.6 & 90.1/88.2 & 92.1/91.7 \\
\rowcolor{mygray}\textbf{PatchGen} & 92.3/93.0 & 96.8/97.1 & 87.8/87.0 & 93.7/93.8 & 92.6/92.7 \\
\midrule
\multicolumn{6}{@{}l}{\textit{Backbone: MAE ViT-B/16}} \\
LP* & 80.8/82.3 & 88.7/86.8 & 81.9/81.8 & 75.1/75.0 & 81.6/81.5 \\
LP + LN finetune (Baseline)* & 84.5/85.4 & 94.1/93.4 & 86.5/85.5 & 87.6/88.3 & 88.2/88.2 \\
\rowcolor{mygray}\textbf{PatchGen} & 86.9/88.3 & 95.6/95.1 & 85.9/86.0 & 89.0/90.4 & 89.3/89.9 \\
\bottomrule
\end{tabular*}
\end{minipage}\hfill
\begin{minipage}[t]{0.49\linewidth}
\vspace{0pt}
\begin{tabular*}{\linewidth}{@{\extracolsep{\fill}}>{\raggedright\arraybackslash}p{88pt}cccc@{}}
\multicolumn{5}{c}{\textbf{(b) HISTOCOLON: ten partially aligned classes}}\\[1pt]
\toprule
Method & CRC-TP & K-16 & K-19 & Avg. \\
\midrule
\multicolumn{5}{@{}l}{\textit{Backbone: CLIP ViT-B/16}} \\
GESTUR* & 39.0/-- & 63.6/-- & 53.7/-- & 52.1/-- \\
LP* & 29.9/29.2 & 49.8/51.7 & 37.4/38.8 & 39.0/39.9 \\
LP + LN finetune (Baseline)* & 37.1/37.7 & 55.8/56.2 & 44.2/43.9 & 45.7/45.9 \\
\rowcolor{mygray}\textbf{PatchGen} & 42.0/42.7 & 58.3/63.9 & 52.8/54.5 & 51.0/53.7 \\
\midrule
\multicolumn{5}{@{}l}{\textit{Backbone: DINOv2 ViT-B/14}} \\
LP* & 49.1/48.8 & 52.4/51.3 & 55.3/55.9 & 52.2/52.0 \\
LP + LN finetune (Baseline)* & 52.5/50.8 & 53.3/52.1 & 52.7/55.5 & 52.8/52.8 \\
Baseline + MIRO* & 51.9/51.5 & 50.5/52.4 & 56.6/57.3 & 53.0/53.7 \\
Baseline + GMDG* & 51.9/51.7 & 53.6/52.4 & 56.6/57.1 & 54.0/53.7 \\
\rowcolor{mygray}\textbf{PatchGen} & 50.7/50.2 & 52.9/54.1 & 59.8/57.9 & 54.4/54.1 \\
\midrule
\multicolumn{5}{@{}l}{\textit{Backbone: DINOv2 ViT-L/14}} \\
LP* & 51.6/53.1 & 52.6/52.2 & 55.5/59.2 & 53.2/54.8 \\
LP + LN finetune (Baseline)* & 51.9/51.7 & 54.6/52.4 & 56.6/57.1 & 54.3/53.7 \\
\rowcolor{mygray}\textbf{PatchGen} & 52.9/53.6 & 51.7/52.8 & 60.1/65.2 & 54.9/57.2 \\
\midrule
\multicolumn{5}{@{}l}{\textit{Backbone: MAE ViT-B/16}} \\
LP* & 37.2/35.7 & 47.6/47.6 & 39.6/39.6 & 41.5/41.0 \\
LP + LN finetune (Baseline)* & 38.9/38.2 & 55.4/57.0 & 39.9/41.6 & 44.7/45.6 \\
\rowcolor{mygray}\textbf{PatchGen} & 43.9/39.0 & 42.3/52.6 & 48.5/45.3 & 44.9/45.6 \\
\bottomrule
\end{tabular*}
\end{minipage}
\vspace{-12pt}
\end{table*}

\textbf{Results.}
\textbf{Matched comparisons.}
Table~\ref{tab:mDG_results} shows that, on natural images with identical backbones and hyperparameters, PatchGen improves average performance over LP+LN and comparable VM-based DG methods (e.g., GMDG). {These improvements establish the effectiveness of the complete module and training objective under unseen-domain evaluation. Table~\ref{tab:ablation} further separates the module contribution from the gain obtained by regularizing it.}
\mbox{\textbf{Reference comparisons.}}
{Table~\ref{tab:mDG_results}} also shows that PatchGen is competitive with prior methods that use different backbones and adaptation protocols. 
Specifically, 
DINOv2 ViT-L/14 + PatchGen achieves 78.9\% without SWAD and 80.9\% with SWAD, the highest reported average in our reference comparison.
With comparable backbones, {Table~\ref{tab:mDG_results_causal}} shows that PatchGen matches or outperforms causal-inspired mDG approaches.
These results support PatchGen’s effectiveness under data shifts, consistent with Prop.~\ref{prop:data_shift}.
\textbf{Sensitivity and robustness analysis.}
Fig.~\ref{fig:analysis} shows stable performance across loss weights and seeds, with and without SWAD, indicating robustness to hyperparameters and initialization. Across seeds $\{0,1,2\}$, average accuracy has standard deviation below $0.3$ percentage points. DINOv2 ViT-L/14 TerraIncognita scores in Table~\ref{tab:mDG_results} use the seed-1 default-weight setting: 66.7/70.2 without/with SWAD. Including the default setting, average accuracy varies within $1.4/1.6$ points across loss weights without/with SWAD.
\textbf{PatchGen with full-model fine-tune.}
Additional LN/full-model fine-tuning comparisons and analysis are provided in Appendix~\ref{app:detailed-results}, Table~\ref{tab:mDG_results_full_finetune}.

\subsection{MDG experiments for histopathological images}

\begin{table}[!t]
\begingroup
\fontsize{7}{8.2}\selectfont
\setlength{\tabcolsep}{1.2pt}\renewcommand{\arraystretch}{0.94}
\setlength{\abovecaptionskip}{3pt}
\setlength{\aboverulesep}{0.22ex}\setlength{\belowrulesep}{0.22ex}

\caption{\textbf{CCD results}: Results of applying our method to PromptCCD with various backbones. \textbf{A/K/U}: all/known/unknown classes (also used in Tables~\ref{tab:avg_mdg_gcd_more_results} and~\ref{tab:gcd_ds}).}
\label{tab:CCD_res}
\centering
{
\begin{tabular}{@{}lccccccccccccc@{}}
\toprule
 \multicolumn{1}{c}{Backbone}&  & \multicolumn{6}{c}{DINO} & \multicolumn{6}{c}{DINOv2} \\
\cmidrule(lr){3-8} \cmidrule(lr){9-14}
\multicolumn{1}{c}{Method} &  & \multicolumn{3}{c}{PromptCCD*} & \multicolumn{3}{c}{+\textbf{PatchGen}} & \multicolumn{3}{c}{PromptCCD*} & \multicolumn{3}{c}{+\textbf{PatchGen}} \\
\midrule
 Datasets & \multicolumn{1}{c}{Session} & \multicolumn{1}{c}{A} & \multicolumn{1}{c}{K} & \multicolumn{1}{c}{U} & \multicolumn{1}{c}{A} & \multicolumn{1}{c}{K} & \multicolumn{1}{c}{U} & \multicolumn{1}{c}{A} & \multicolumn{1}{c}{K} & \multicolumn{1}{c}{U} & \multicolumn{1}{c}{A} & \multicolumn{1}{c}{K} & \multicolumn{1}{c}{U} \\
\cmidrule(lr){3-5} \cmidrule(lr){6-8} \cmidrule(lr){9-11} \cmidrule(lr){12-14}
CIFAR-100 & 1 & 72.0 & 81.8 & 65.2 & \cellcolor{mygray}75.3 & \cellcolor{mygray}84.0 & \cellcolor{mygray}69.2 & 77.2 & 89.1 & 68.9 & \cellcolor{mygray}76.3 & \cellcolor{mygray}86.9 & \cellcolor{mygray}68.9 \\
 & 2 & 57.7 & 77.9 & 53.8 & \cellcolor{mygray}62.6 & \cellcolor{mygray}74.6 & \cellcolor{mygray}60.3 & 59.6 & 78.6 & 55.9 & \cellcolor{mygray}73.1 & \cellcolor{mygray}82.8 & \cellcolor{mygray}71.3 \\
 & 3 & 48.2 & 73.2 & 43.8 & \cellcolor{mygray}47.9 & \cellcolor{mygray}72.8 & \cellcolor{mygray}43.5 & 49.9 & 75.0 & 45.5 & \cellcolor{mygray}49.6 & \cellcolor{mygray}74.9 & \cellcolor{mygray}45.2 \\
\textbf{} & Avg. & 59.3 & 77.6 & 54.3 & \cellcolor{mygray}61.9 & \cellcolor{mygray}77.1 & \cellcolor{mygray}57.7 & 62.2 & 80.9 & 56.8 & \cellcolor{mygray}66.4 & \cellcolor{mygray}81.5 & \cellcolor{mygray}61.8 \\
\midrule
CUB & 1 & 56.5 & 76.1 & 43.4 & \cellcolor{mygray}55.7 & \cellcolor{mygray}73.6 & \cellcolor{mygray}43.7 & 70.4 & 89.3 & 57.8 & \cellcolor{mygray}70.5 & \cellcolor{mygray}88.2 & \cellcolor{mygray}58.7 \\
 & 2 & 46.3 & 70.0 & 41.2 & \cellcolor{mygray}50.3 & \cellcolor{mygray}72.9 & \cellcolor{mygray}45.5 & 64.1 & 81.4 & 60.5 & \cellcolor{mygray}69.6 & \cellcolor{mygray}84.3 & \cellcolor{mygray}66.5 \\
 & 3 & 56.3 & 77.9 & 52.1 & \cellcolor{mygray}59.9 & \cellcolor{mygray}75.0 & \cellcolor{mygray}57.0 & 68.2 & 82.1 & 65.5 & \cellcolor{mygray}66.4 & \cellcolor{mygray}84.3 & \cellcolor{mygray}63.0 \\
\textbf{} & Avg. & 53.0 & 74.6 & 45.6 & \cellcolor{mygray}55.3 & \cellcolor{mygray}73.8 & \cellcolor{mygray}48.7 & 67.6 & 84.3 & 61.2 & \cellcolor{mygray}68.9 & \cellcolor{mygray}85.6 & \cellcolor{mygray}62.7 \\
\bottomrule
\end{tabular}%
}
\par\endgroup
\vspace{-12pt}
\end{table}

\begin{table*}[!t]
\centering
\caption{\textbf{MDG+GCD results}: Accuracy scores of each dataset and accuracy averaged over all datasets.}
\label{tab:avg_mdg_gcd_more_results}
\label{tab:avg_mdg_gcd_results}
{\fontsize{8}{9}\selectfont
\setlength{\tabcolsep}{3pt}\renewcommand{\arraystretch}{0.98}
\setlength{\aboverulesep}{0.5pt}\setlength{\belowrulesep}{0.7pt}
\resizebox{\linewidth}{!}{%
\begin{tabular}{@{}l*{3}{r}@{\hspace{9pt}}*{3}{r}@{\hspace{9pt}}*{3}{r}@{\hspace{9pt}}*{3}{r}@{\hspace{9pt}}*{3}{r}@{\hspace{10pt}}*{3}{r}@{}}
\toprule
\textbf{Method} & \multicolumn{3}{c}{\textbf{PACS}} & \multicolumn{3}{c}{\textbf{OfficeHome}} & \multicolumn{3}{c}{\textbf{VLCS}} & \multicolumn{3}{c}{\textbf{TerraIncognita}} & \multicolumn{3}{c}{\textbf{DomainNet}} & \multicolumn{3}{c}{\textbf{Avg.}}\\
\cmidrule(lr){2-4}\cmidrule(lr){5-7}\cmidrule(lr){8-10}\cmidrule(lr){11-13}\cmidrule(lr){14-16}\cmidrule(lr){17-19}
 & A & K & U & A & K & U & A & K & U & A & K & U & A & K & U & A & K & U\\
\midrule
R: ERM~\cite{vapnik1999overview} & 57.26 & 77.77 & 22.33 & 44.80 & 74.67 & 8.50 & 61.51 & 82.89 & 34.88 & 37.34 & 20.46 & 45.15 & 22.56 & 40.89 & 6.85 & 44.69 & 59.34 & 23.54\\
R: PIM~\cite{chiaroni2023parametric} & 56.35 & 71.06 & 27.43 & 43.42 & 72.44 & 8.13 & 63.19 & 80.34 & 40.24 & 47.75 & 35.31 & 50.85 & 24.03 & 42.59 & 7.86 & 46.95 & 60.35 & 26.90\\
R: MIRO~\cite{cha2022miro} & 56.83 & 85.62 & 24.85 & 48.28 & 80.61 & 9.03 & 61.53 & 82.72 & 35.03 & 50.22 & 39.92 & 49.45 & 31.49 & 55.44 & 10.57 & 49.67 & 68.86 & 25.79\\
R: GMDG~\cite{tan2024rethinking} & 58.33 & 91.46 & 10.18 & 48.85 & 81.41 & 9.22 & 61.36 & 83.31 & 33.75 & 40.02 & 32.38 & 40.07 & {31.15} & 55.17 & 10.18 & 47.94 & 68.75 & 20.68\\
R: L-Reg~\cite{tan2024interpret} & 67.82 & 91.86 & 31.33 & 51.96 & 79.74 & 18.15 & 62.32 & 82.77 & 36.09 & 45.86 & 39.77 & 41.55 & 31.75 & 55.18 & 11.30 & 51.94 & 69.86 & 27.68\\
R: GMDG+PL-Reg~\cite{tan2025saving} & 74.21 & 92.07 & 42.69 & 54.73 & 80.98 & 22.58 & 65.10 & 80.13 & 45.87 & 48.93 & 31.65 & 54.28 & 32.54 & 55.53 & 12.37 & 55.10 & 68.07 & 35.56\\
B: Baseline* & 96.18 & 96.38 & 96.45 & 79.26 & 87.52 & 72.02 & 85.22 & 87.45 & 82.52 & 63.32 & 58.80 & 67.82 & 56.44 & 56.96 & 56.38 & 76.08 & 77.42 & 75.04\\
B: L-Reg* & 96.13 & 95.63 & 96.98 & 81.36 & 91.72 & 72.11 & 84.00 & 86.46 & 80.89 & 59.54 & 47.36 & 67.69 & 55.74 & 55.76 & 56.11 & 75.35 & 75.39 & 74.76\\
\rowcolor{mygray}B: \textbf{PatchGen} & 96.48 & 96.81 & 96.28 & 82.55 & 92.84 & 72.97 & 85.66 & 88.39 & 82.33 & 63.87 & 53.98 & 69.85 & 56.38 & 56.68 & 56.53 & 76.99 & 77.74 & 75.59\\
L: Baseline* & 96.82 & 97.73 & 96.21 & 79.49 & 85.02 & 74.44 & 88.78 & 90.38 & 86.89 & 68.55 & 58.79 & 76.11 & 59.53 & 59.78 & 59.73 & 78.63 & 78.34 & 78.68\\
\rowcolor{mygray}L: \textbf{PatchGen} & 97.11 & 97.50 & 96.91 & 82.92 & 91.84 & 74.98 & 89.08 & 90.92 & 86.90 & 68.66 & 61.46 & 75.34 & 59.85 & 59.95 & 60.13 & 79.52 & 80.33 & 78.85\\
M: Baseline* & 69.79 & 72.74 & 62.67 & 76.91 & 87.71 & 65.60 & 53.93 & 59.77 & 46.86 & 40.84 & 28.70 & 49.74 & 26.47 & 25.84 & 27.11 & 53.59 & 54.95 & 50.40\\
\rowcolor{mygray}M: \textbf{PatchGen} & 70.73 & 74.86 & 61.37 & 76.24 & 85.11 & 65.97 & 54.97 & 59.82 & 49.11 & 45.41 & 24.28 & 59.30 & 27.16 & 26.32 & 27.93 & 54.90 & 54.08 & 52.74\\
\bottomrule
\end{tabular}%
}
\par\vspace{2pt}{\raggedright Backbone: R=RegNetY-16GF; B=DINOv2 ViT-B/14; L=DINOv2 ViT-L/14; M=MAE ViT-B/16.\par}
}
\vspace{-14pt}
\end{table*}

\textbf{Datasets.}
HISTOPANTUM~\cite{zamanitajeddin2024benchmarking} has four organ domains sharing normal/tumor labels; residual shifts may involve organ appearance and class priors. HISTOCOLON combines CRC-TP~\cite{javed2020cellular}, K-16~\cite{kather2016multi}, and K-19~\cite{kather2019predicting}, with ten partially shared classes.
Both use leave-one-domain-out evaluation and source-only model selection. In HISTOCOLON, one dataset is test-only and the other two supply training/validation data, with no test-dataset image, slide, or patient in either source split.
\textbf{Experimental settings.}
We use pretrained DINOv2~\cite{oquab2023DINOv2}/MAE~\cite{he2022masked} ViT backbones and LP/LP+LN controls as in the natural-image setup; MIRO~\cite{cha2022miro} and GMDG~\cite{tan2024rethinking} use DINOv2 ViT-B/14. DPR~\cite{cheng2024disentangled} and CLIPCEIL++~\cite{yu2024clipceil} refine text--image alignment; our CLIP~\cite{radford2021learning} comparisons use GESTUR~\cite{lew2023gestur} (no downstream alignment) and PatchGen on CLIP. We match shared settings when possible and use official method-specific settings otherwise. We set $\lambda_{ms}\!=\!\lambda_{sim}\!=\!1e^{-3}$ and apply SWAD~\cite{cha2021swad} where available.

\textbf{Results.}
PatchGen improves matched-baseline averages in most settings (Table~\ref{tab:mdg_histopathology}).
Specifically, without SWAD, DINOv2 ViT-B/14 achieves $92.8\%$ on HISTOPANTUM and $54.4\%$ on HISTOCOLON, respectively, versus $91.0\%$ and $52.1\%$ for GESTUR with CLIP ViT-B/16.
Although primarily designed for purely visual models, where image features are well preserved, it also yields clear gains when applied to CLIP.
PatchGen substantially improves the matched CLIP LP+LN baseline and remains
competitive with GESTUR on HISTOCOLON. 
\textbf{Histopathology visualizations.}
In mDG (Fig.~\ref{fig:path_exp1}(a)), PatchGen highlights tumor-consistent tissue; baselines more often emphasize surrounding context. A pathology expert assessed their morphological plausibility.

\subsection{Experiments on CCD tasks}

\textbf{Experimental settings.}
We reproduce PromptCCD's official code and protocol~\cite{cendra2024promptccd} on CIFAR-100~\cite{krizhevsky2009learning} and CUB~\cite{wah2011caltech} with DINO ViT-B/16 and DINOv2 ViT-B/14, adding PatchGen within the same framework. Methods share settings where possible; otherwise, official method-specific settings apply. We use $\lambda_{ms}=\lambda_{sim}=10^{-3}$ throughout.

\textbf{Results.}
PatchGen improves average unknown-class accuracy in all four dataset--backbone settings by $1.5$--$5.0$ points (Table~\ref{tab:CCD_res}), reaching a $15.4$-point gain in CIFAR-100 session 2 with DINOv2. All-class averages also improve, with smaller changes in known-class accuracy. These results support the proxy's utility for unknown-class generalization, consistent with Prop.~\ref{prop:target_shift}'s class-separation motivation.
\subsection{Experiments on mDG+GCD tasks}
\textbf{Experimental settings.} Following L-Reg~\cite{tan2024interpret} and its extension PL-Reg~\cite{tan2025saving}, we adapt the five natural-image mDG datasets to mDG+GCD. Half the classes are unknown.
\begingroup
\setlength{\intextsep}{0pt}
Evaluation follows L-Reg, using multiple-backbone baselines including L-Reg on DINOv2 ViT-B/14. We set $\lambda_{ms}=\lambda_{sim}=0.001$.

\textbf{Results.}
PatchGen improves average all-class and unknown-class accuracy across backbones (Table~\ref{tab:avg_mdg_gcd_more_results}). With DINOv2 ViT-L/14, these increase from \textbf{A}:78.63\% \textbf{K}:78.34\% \textbf{U}:78.68\% to \textbf{A}:79.52\% \textbf{K}:80.33\% \textbf{U}: 78.85\%, supporting generalization under joint domain and label-space shifts.

\subsection{Apply to category-discovery frameworks}
\label{sec:gcd_ds}

\begingroup
\setlength{\intextsep}{0pt}

\begin{table}[!t]
\centering
\setlength{\abovecaptionskip}{0pt}
\caption{\textbf{GCD results:} Accuracy (\%) on CUB~\cite{wah2011caltech} and SCars~\cite{krause20133d}.}
\label{tab:cub_gcd}
\fontsize{7.5}{8.5}\selectfont
\setlength{\tabcolsep}{1pt}
\renewcommand{\arraystretch}{0.95}
\setlength{\aboverulesep}{0.5pt}
\setlength{\belowrulesep}{0.5pt}
\begin{minipage}[t]{0.49\linewidth}
\vspace{0pt}
\begin{tabular*}{\linewidth}{@{\extracolsep{\fill}}lrrr@{}}
\toprule
SimGCD & A & K & U \\
\midrule
Baseline$^*$(CUB) & 61.92 & 66.38 & 59.69 \\
\rowcolor{mygray}
\textbf{+PatchGen}(CUB)
& \textbf{63.41} & \textbf{68.91} & \textbf{60.66} \\
\midrule
Baseline$^*$(SCars) & 51.06 & 69.57 & 42.12 \\
\rowcolor{mygray}
\textbf{+PatchGen}(SCars)
& \textbf{53.42} & \textbf{70.01} & \textbf{45.40} \\
\bottomrule
\end{tabular*}
\end{minipage}\hfill
\begin{minipage}[t]{0.49\linewidth}
\vspace{0pt}
\begin{tabular*}{\linewidth}{@{\extracolsep{\fill}}lrrr@{}}
\toprule
SimGCD+AF & A & K & U \\
\midrule
Baseline$^*$(CUB) & 65.95 & 72.31 & 62.76 \\
\rowcolor{mygray}
\textbf{+PatchGen}(CUB)
& \textbf{66.88} & \textbf{72.78} & \textbf{63.93} \\
\midrule
Baseline$^*$(SCars) & 61.78 & 77.91 & 53.99 \\
\rowcolor{mygray}
\textbf{+PatchGen}(SCars)
& \textbf{62.86} & \textbf{78.86} & \textbf{55.13} \\
\bottomrule
\end{tabular*}
\end{minipage}
\vspace{-6pt}
\end{table}

\begin{table}[!t]
\centering
\setlength{\abovecaptionskip}{0pt}
\caption{\textbf{GCD+DS results:} Accuracy (\%) on corrupted domains.}
\label{tab:gcd_ds}
{\fontsize{7.5}{8.5}\selectfont
\setlength{\tabcolsep}{1pt}
\renewcommand{\arraystretch}{0.95}
\setlength{\aboverulesep}{0.5pt}
\setlength{\belowrulesep}{0.5pt}
\begin{minipage}[t]{0.49\linewidth}
\vspace{0pt}
\begin{tabular*}{\linewidth}{@{\extracolsep{\fill}}lrrr@{}}
\toprule
CUB-C & A & K & U \\
\midrule
HiLo$^\dagger$ & 51.6 & 52.1 & 51.2 \\
\rowcolor{mygray}\textbf{HiLo+PatchGen} & \textbf{53.1} & \textbf{54.8} & \textbf{51.4} \\
\bottomrule
\end{tabular*}\end{minipage}\hfill
\begin{minipage}[t]{0.49\linewidth}
\vspace{0pt}
\begin{tabular*}{\linewidth}{@{\extracolsep{\fill}}lrrr@{}}
\toprule
Scar-C & A & K & U \\
\midrule
HiLo$^\dagger$ & 35.1 & 42.8 & 27.7 \\
\rowcolor{mygray}\textbf{HiLo+PatchGen} & \textbf{37.0} & \textbf{46.0} & \textbf{28.3} \\
\bottomrule
\end{tabular*}\end{minipage}}
\vspace{-14pt}
\end{table}

\mbox{\textbf{Experimental settings.}}
Detailed protocols are provided in Appendix~\ref{app:exp-results}.
\textbf{Generalized category discovery (GCD)}~\cite{vaze2022generalized} uses labeled examples from known classes to recognize both known and novel classes in unlabeled data.
We test PatchGen on CUB using SimGCD and SimGCD+AF following their settings~\cite{wen2023simgcd,xu2025distracted}.
In \textbf{GCD with domain shift (GCD+DS)}, the model needs to conduct GCD using labeled clean known-class images and unlabeled clean/corrupted images from known and unknown classes.
Follow HiLo,
HiLo+PatchGen is evaluated on CUB-C and Scar-C~\cite{wang2025hilo}\footnote{
$^*$: reproduced results using their official code.
$\dagger$: reported HiLo results from~\cite{tan2025saving}, used as external references. PatchGen is evaluated locally.}.

\textbf{Results.}
Table~\ref{tab:cub_gcd}
and Table~\ref{tab:gcd_ds} show that PatchGen further improves generalization of multiple methods under broader generalization scenarios. 
PatchGen improves A/K/U for both GCD baselines on CUB and SCars, with All gains of $0.93$--$2.36$ points.
PatchGen also enhances reported HiLo on both datasets.
All accuracy increases by $1.5/1.9$ percentage points,
while Unknown accuracy increases by $0.2/0.6$ points
on CUB-C/Scar-C, respectively.

\par\endgroup

\subsection{Patch-selection design analysis}
\label{sec:ablation}

The task numbers below refer to the row IDs in Table~\ref{tab:ablation}.

\textbf{The first block} (Tasks~1.1--1.2) presents baseline methods.
\textbf{The second block} (Tasks~2.1--2.4) evaluates alternative design choices to justify our formulation:
$\bullet$ \textbf{B + $\mathcal{M}_{s}$}: Baseline with a learnable patch mask $\mathcal{M}_{s}$ shared across all samples (static).
$\bullet$ \textbf{B + $\mathcal{M}_{d}$}: Baseline with a learnable dimension-wise mask on feature channels.
$\bullet$ \textbf{B + $Attn$}: Baseline with an additional standard attention module.
$\bullet$ \textbf{B + $Attn$ + $L_{ms}$}: Attention-based variant further optimized with our mask-sharpening loss $\mathcal L_{ms}$, where attention maps are normalized to $[0,1]$ before applying the loss.
For this variant, the scalar patch score is computed from the pre-softmax attention logits using the same outgoing-score reduction as Eq.~\eqref{eq:predictive_patch_mask}, followed by min--max normalization over patches within each image before applying $\mathcal L_{ms}$.
\textbf{The third block} (Tasks~3.1--3.5) presents ablations of our proposed components, where $\mathcal M_p$ denotes the sample-dependent mask module that outputs
$\mathbf m_\phi(X)$, and $m_{\phi,p}(X)$ denotes the weight assigned to patch
$p$.
\textbf{$B+\mathcal{M}_p$} (Task~3.1) is an equal-parameter selector control: it uses the same sample-dependent mask module as PatchGen but omits the predictive-subset constraints. This comparison helps isolate the proposed constraints’ effect from added trainable capacity.

\textbf{Comparison with other masking strategies.}
Compared with static patch masking ($\mathcal{M}_{s}$) and dimension-wise masking ($\mathcal{M}_{d}$), our learned patch-wise masking $\mathcal{M}_{p}$ yields higher average accuracy in both SWAD settings, although the gain is not uniform for every individual dataset.
This suggests that sample-adaptive patch modeling is more effective than globally shared or channel-level masking at learning a
useful soft proxy for the oracle intra-image predictive subset.
\textbf{Comparison with standard attention.}
The sample-dependent mask module is competitive with standard attention
without SWAD and provides a small average advantage with SWAD. Applying
$\mathcal L_{ms}$ directly to standard attention does not improve its average
performance, indicating that low-score suppression alone is insufficient.
The strongest results are obtained when the proposed selector is combined
with mask suppression, selected-confidence regularization, and
class-conditional feature alignment.
\textbf{Ablation of proposed components.}
Tasks~3.1 and~3.3--3.5 form the principal incremental ablation, whereas
Task~3.2 evaluates an increased weight on $\mathcal L_{ms}$.
The selected-confidence and class-conditional objectives provide modest
additional average gains, and the complete model achieves the strongest
overall result, although individual intermediate variants do not improve monotonically.
This suggests that the objectives are most effective when combined.

\begin{table}[!t]
\centering
\setlength{\abovecaptionskip}{2pt}
\caption{\textbf{Ablation results} across five mDG datasets. An Avg. comparison is shown in Appendix~\ref{app:detailed-results}, Fig.~\ref{fig:ablation}.}
\label{tab:ablation}
{\fontsize{7}{8}\selectfont
\setlength{\tabcolsep}{1.5pt}
\renewcommand{\arraystretch}{0.94}
\setlength{\aboverulesep}{0.15ex}
\setlength{\belowrulesep}{0.15ex}
\resizebox{\linewidth}{!}{%
\begin{tabular}{@{}c>{\raggedright\arraybackslash}p{56pt}cccccc@{\hspace{7pt}}cccccc@{}}
\toprule
& & \multicolumn{6}{c}{\textbf{Without SWAD}} & \multicolumn{6}{c}{\textbf{With SWAD}} \\
\cmidrule(lr){3-8}\cmidrule(lr){9-14}
ID & Method & P & V & O & T & D & Avg. & P & V & O & T & D & Avg. \\
\midrule
1.1 & LP+LN (B) & 97.0 & 78.6 & 86.7 & 65.7 & 54.3 & 76.5 & 97.5 & 81.0 & 89.7 & 67.4 & 58.1 & 78.7 \\
1.2 & B + GMDG & 96.1 & 79.2 & 85.7 & 67.9 & 54.0 & 76.6 & 96.8 & 81.1 & 89.5 & 68.5 & 57.7 & 78.7 \\
\midrule
2.1 & B + $\mathcal{M}_{s}$ & 96.8 & 81.4 & 85.3 & 64.7 & 57.4 & 77.1 & 97.4 & 81.4 & 88.9 & 66.8 & 60.1 & 78.9 \\
2.2 & B + $\mathcal{M}_{d}$ & 96.7 & 80.9 & 86.3 & 65.9 & 56.7 & 77.3 & 97.7 & 81.1 & 88.9 & 67.1 & 60.5 & 79.1 \\
2.3 & B + $Attn$ & 96.8 & 82.1 & 86.7 & 67.3 & 56.9 & 78.0 & 97.5 & 82.0 & 89.1 & 69.2 & 60.5 & 79.7 \\
2.4 & 2.3 + ${L}_{ms}$ & 97.8 & 81.9 & 86.4 & 65.8 & 57.7 & 77.9 & 97.8 & 81.6 & 89.6 & 69.0 & 60.5 & 79.7 \\
\midrule
3.1 & B + $\mathcal{M}_{p}$ & 96.7 & 82.2 & 86.6 & 66.5 & 57.3 & 77.9 & 97.9 & 82.4 & 89.7 & 69.0 & 60.6 & 79.9 \\
3.2 & 3.1 + $10{L}_{ms}$ & 96.9 & 81.5 & 86.7 & 66.2 & 57.9 & 77.8 & 97.5 & 83.3 & 90.0 & 70.2 & 60.9 & 80.4 \\
3.3 & 3.1 + ${L}_{ms}$ & 96.8 & 83.4 & 86.8 & 68.4 & 58.4 & 78.7 & 97.0 & 83.1 & 90.0 & 70.6 & 62.2 & 80.6 \\
3.4 & 3.3 + $\mathcal L_{\mathrm{conf}}$ & 97.5 & 82.6 & 87.4 & 67.4 & 58.9 & 78.8 & 97.5 & 83.0 & 90.2 & 70.5 & 62.5 & 80.7 \\
\rowcolor{mygray}
3.5 & \mbox{3.4 + ${L}_{sim}$ (\textbf{Ours})} & 97.1 & 83.6 & 88.0 & 66.7 & 59.2 & 78.9 & 97.6 & 83.4 & 90.6 & 70.2 & 62.7 & 80.9 \\
\rowcolor{mygray}
\multicolumn{2}{l}{\textbf{Gain over B}} & +0.1 & +5.0 & +1.3 & +1.0 & +4.9 & +2.4 & +0.1 & +2.4 & +0.9 & +2.8 & +4.6 & +2.2 \\
\bottomrule
\end{tabular}%
}}
\vspace{-10pt}
\end{table}

\subsection{Patch-selection diagnostics and efficiency}
\label{sec:patch_diag}

\begin{figure}[!t]
    \centering
    \includegraphics[width=\linewidth]{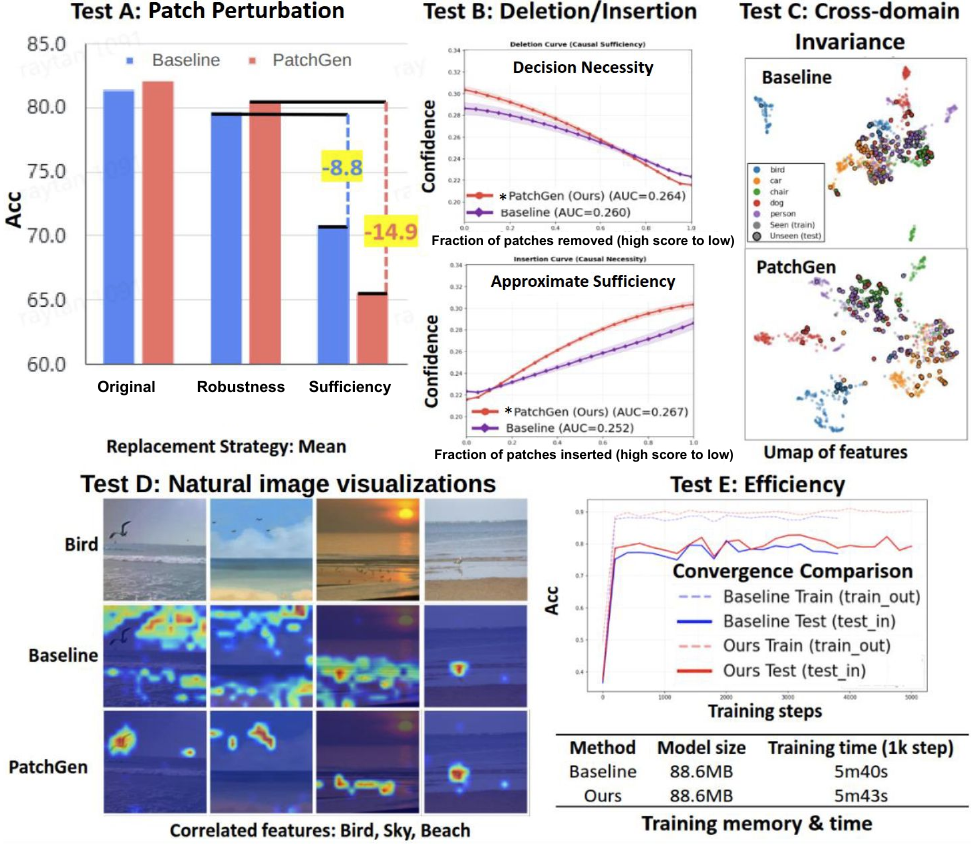}
    \vspace{-8pt}
\caption{\textbf{Patch-selection diagnostics} on VLCS (SUN09 unseen) with DINOv2-ViT-S/14. 
    (A)~Patch perturbation: keeping selected patches is robust, while replacing selected patches causes a sharp accuracy drop. (B)~Insertion/deletion AUC curves. (C)~UMAP of baseline vs. PatchGen features showing tighter cross-domain clusters. (D)~Visualization on natural images: PatchGen focuses on label-relevant objects while the baseline is distracted by spurious context. (E)~Training efficiency comparison.
     }
    \label{fig:patch_diagnostics}
\vspace{-10pt}
\end{figure}

Fig.~\ref{fig:patch_diagnostics} examines decision relevance and efficiency on VLCS with SUN09 unseen, using DINOv2 ViT-S/14 under mDG.

\textbf{A. Patch perturbation test.}
Let $\widehat C_\phi(X)=\{p:m_{\phi,p}(X)\ge 0.5\}$ denote the thresholded selected patch set, and let
$\widehat R_\phi(X)=[P]\setminus\widehat C_\phi(X)$ denote its thresholded complement.
{The same patch-replacement operator is used for selected and complementary sets across methods.}
We assess approximate sufficiency by preserving $\widehat C_\phi(X)$ while
perturbing $\widehat R_\phi(X)$, and decision necessity by 
perturbing $\widehat C_\phi(X)$ while preserving
$\widehat R_\phi(X)$.
The resulting accuracy drops quantify the classifier’s decision dependence under this protocol.
\textbf{B. Deletion/insertion.}
Patches are ranked by each method and progressively inserted or deleted.
We report insertion AUC ($\uparrow$; faster confidence recovery is
better) and deletion AUC ($\downarrow$; faster confidence removal
is better). PatchGen improves insertion AUC from 0.252 to 0.267, but slightly worsens deletion AUC from 0.260 to 0.264 (Fig.~\ref{fig:patch_diagnostics}(B)). 
\textbf{C. Cross-domain representation geometry.}
UMAP shows tighter class clusters with less domain  separation, consistent with reduced domain-related variation.
\textbf{D. Natural image visualizations.}
PatchGen's selection maps show a preference for object-relevant regions over surrounding context; perturbation tests measure its effect on the trained classifier.
\textbf{E. Computational efficiency.}
On an NVIDIA RTX~3090, 1,000 training steps take 5~min~40~s (LP+LN) vs. 5~min~43~s (PatchGen), a ~0.9\% increase. Fig.~\ref{fig:patch_diagnostics}(E) reports the same rounded model size (88.6~MB) for both.
\textbf{Summary.}
Taken together, these diagnostics support the decision relevance of the learned patch scores for the trained classifier, while insertion and deletion yield mixed outcomes.

\section{Discussion and conclusion}

\textbf{Conclusion.}
We theoretically motivate an oracle intra-image predictive subset and empirically evaluate its learned soft proxy for visual generalization. Across the evaluated data-, target-, and all-shift benchmarks, PatchGen yields competitive or improved performance relative to matched-backbone baselines while providing sample-dependent patch-selection maps and decision-relevance diagnostics.
\textbf{Limitations.}
The quantitative correspondence between the learned mask and oracle-selected patches remains to be evaluated directly. Category-discovery performance depends on host pseudo-label quality: correct assignments support same-class alignment, while incorrect ones may mix distinct classes through $\mathcal L_{\mathrm{sim}}$. Vision-language pretraining may already emphasize class-semantic regions, making additional patch-selection gains backbone-dependent.

\section*{LLM Usage Statement}
OpenAI Codex assisted with language refinement and LaTeX formatting, with additional support for scripts used to process and plot existing experimental results. The authors retain full responsibility for all reported results, scientific claims, and the final manuscript.

\FloatBarrier

{\small
\bibliographystyle{ieeenat_fullname}
\bibliography{main}

@String(CVPR  = {IEEE Conf. Comput. Vis. Pattern Recog.})

@String(ICCV  = {Int. Conf. Comput. Vis.})

@String(ECCV  = {Eur. Conf. Comput. Vis.})

@String(AAAI  = {AAAI})

@String(ICME  = {Int. Conf. Multimedia and Expo})

@String(CVPR  = {CVPR})

@String(ICCV  = {ICCV})

@String(ECCV  = {ECCV})

@String(ICME  =	{ICME})

@article{xu2025distracted,
  title={A Hidden Stumbling Block in Generalized Category Discovery: Distracted Attention},
  author={Qiyu Xu and Zhanxuan Hu and Yu Duan and Ercheng Pei and Yong-Hang Tai},
  journal={2025 IEEE/CVF International Conference on Computer Vision (ICCV)},
  year={2025},
  pages={405-414}
}

@inproceedings{wen2023simgcd,
    author    = {Wen, Xin and Zhao, Bingchen and Qi, Xiaojuan},
    title     = {Parametric Classification for Generalized Category Discovery: A Baseline Study},
    booktitle = {Proceedings of the IEEE/CVF International Conference on Computer Vision (ICCV)},
    year      = {2023},
    pages     = {16590-16600}
}

@article{tan2024interpret,
  title={Interpret your decision: Logical reasoning regularization for generalization in visual classification},
  author={Tan, Zhaorui and Yang, Xi and Wang, Qiufeng and Nguyen, Anh and Huang, Kaizhu},
  journal={Advances in Neural Information Processing Systems},
  volume={37},
  pages={18166--18204},
  year={2024}
}

@article{abbe2024generalization,
  title={Generalization on the unseen, logic reasoning and degree curriculum},
  author={Abbe, Emmanuel and Bengio, Samy and Lotfi, Aryo and Rizk, Kevin},
  journal={Journal of Machine Learning Research},
  volume={25},
  number={331},
  pages={1--58},
  year={2024}
}

@inproceedings{tan2024rethinking,
  title={Rethinking multi-domain generalization with a general learning objective},
  author={Tan, Zhaorui and Yang, Xi and Huang, Kaizhu},
  booktitle={proceedings of the IEEE/CVF conference on computer vision and pattern recognition},
  pages={23512--23522},
  year={2024}
}

@inproceedings{chiaroni2023parametric,
  title={Parametric information maximization for generalized category discovery},
  author={Chiaroni, Florent and Dolz, Jose and Masud, Ziko Imtiaz and Mitiche, Amar and Ben Ayed, Ismail},
  booktitle={Proceedings of the IEEE/CVF international conference on computer vision},
  pages={1729--1739},
  year={2023}
}

@inproceedings{vaze2022generalized,
  title={Generalized category discovery},
  author={Vaze, Sagar and Han, Kai and Vedaldi, Andrea and Zisserman, Andrew},
  booktitle={Proceedings of the IEEE/CVF conference on computer vision and pattern recognition},
  pages={7492--7501},
  year={2022}
}

@techreport{wah2011caltech,
  title={The caltech-ucsd birds-200-2011 dataset},
  author={Wah, Catherine and Branson, Steve and Welinder, Peter and Perona, Pietro and Belongie, Serge and others},
  year={2011},
  institution={Technical Report CNS-TR-2011-001, California Institute of Technology}
}

@inproceedings{krause20133d,
  title={3d object representations for fine-grained categorization},
  author={Krause, Jonathan and Stark, Michael and Deng, Jia and Fei-Fei, Li},
  booktitle={Proceedings of the IEEE international conference on computer vision workshops},
  pages={554--561},
  year={2013}
}

@inproceedings{cha2022miro,
  title={Domain generalization by mutual-information regularization with pre-trained models},
  author={Cha, Junbum and Lee, Kyungjae and Park, Sungrae and Chun, Sanghyuk},
  booktitle={European conference on computer vision},
  pages={440--457},
  year={2022},
  organization={Springer}
}

@inproceedings{li2018domain,
  title={Domain generalization via conditional invariant representations},
  author={Li, Ya and Gong, Mingming and Tian, Xinmei and Liu, Tongliang and Tao, Dacheng},
  booktitle={Proceedings of the AAAI conference on artificial intelligence},
  volume={32},
  number={1},
  year={2018}
}

@article{gulrajani2020search,
  title={In search of lost domain generalization},
  author={Gulrajani, Ishaan and Lopez-Paz, David},
  journal={arXiv preprint arXiv:2007.01434},
  year={2020}
}

@inproceedings{hu2020domain,
  title={Domain generalization via multidomain discriminant analysis},
  author={Hu, Shoubo and Zhang, Kun and Chen, Zhitang and Chan, Laiwan},
  booktitle={Uncertainty in artificial intelligence},
  pages={292--302},
  year={2020},
  organization={PMLR}
}

@article{ganin2016domain,
  title={Domain-adversarial training of neural networks},
  author={Ganin, Yaroslav and Ustinova, Evgeniya and Ajakan, Hana and Germain, Pascal and Larochelle, Hugo and Laviolette, Fran{\c{c}}ois and March, Mario and Lempitsky, Victor},
  journal={Journal of machine learning research},
  volume={17},
  number={59},
  pages={1--35},
  year={2016}
}

@inproceedings{sun2016deep,
  title={Deep coral: Correlation alignment for deep domain adaptation},
  author={Sun, Baochen and Saenko, Kate},
  booktitle={European conference on computer vision},
  pages={443--450},
  year={2016},
  organization={Springer}
}

@inproceedings{li2018deep,
  title={Deep domain generalization via conditional invariant adversarial networks},
  author={Li, Ya and Tian, Xinmei and Gong, Mingming and Liu, Yajing and Liu, Tongliang and Zhang, Kun and Tao, Dacheng},
  booktitle={Proceedings of the European conference on computer vision (ECCV)},
  pages={624--639},
  year={2018}
}

@article{bui2021exploiting,
  title={Exploiting domain-specific features to enhance domain generalization},
  author={Bui, Manh-Ha and Tran, Toan and Tran, Anh and Phung, Dinh},
  journal={Advances in Neural Information Processing Systems},
  volume={34},
  pages={21189--21201},
  year={2021}
}

@inproceedings{li2017deeper,
  title={Deeper, broader and artier domain generalization},
  author={Li, Da and Yang, Yongxin and Song, Yi-Zhe and Hospedales, Timothy M},
  booktitle={Proceedings of the IEEE international conference on computer vision},
  pages={5542--5550},
  year={2017}
}

@inproceedings{fang2013unbiased,
  title={Unbiased metric learning: On the utilization of multiple datasets and web images for softening bias},
  author={Fang, Chen and Xu, Ye and Rockmore, Daniel N},
  booktitle={Proceedings of the IEEE international conference on computer vision},
  pages={1657--1664},
  year={2013}
}

@inproceedings{venkateswara2017deep,
  title={Deep hashing network for unsupervised domain adaptation},
  author={Venkateswara, Hemanth and Eusebio, Jose and Chakraborty, Shayok and Panchanathan, Sethuraman},
  booktitle={Proceedings of the IEEE conference on computer vision and pattern recognition},
  pages={5018--5027},
  year={2017}
}

@inproceedings{beery2018recognition,
  title={Recognition in terra incognita},
  author={Beery, Sara and Van Horn, Grant and Perona, Pietro},
  booktitle={Proceedings of the European conference on computer vision (ECCV)},
  pages={456--473},
  year={2018}
}

@inproceedings{peng2019moment,
  title={Moment matching for multi-source domain adaptation},
  author={Peng, Xingchao and Bai, Qinxun and Xia, Xide and Huang, Zijun and Saenko, Kate and Wang, Bo},
  booktitle={Proceedings of the IEEE/CVF international conference on computer vision},
  pages={1406--1415},
  year={2019}
}

@article{cha2021swad,
  title={Swad: Domain generalization by seeking flat minima},
  author={Cha, Junbum and Chun, Sanghyuk and Lee, Kyungjae and Cho, Han-Cheol and Park, Seunghyun and Lee, Yunsung and Park, Sungrae},
  journal={Advances in Neural Information Processing Systems},
  volume={34},
  pages={22405--22418},
  year={2021}
}

@article{arjovsky2019invariant,
  title={Invariant risk minimization},
  author={Arjovsky, Martin and Bottou, L{\'e}on and Gulrajani, Ishaan and Lopez-Paz, David},
  journal={arXiv preprint arXiv:1907.02893},
  year={2019}
}

@article{vapnik1999overview,
  title={An overview of statistical learning theory},
  author={Vapnik, Vladimir N},
  journal={IEEE transactions on neural networks},
  volume={10},
  number={5},
  pages={988--999},
  year={1999},
  publisher={IEEE}
}

@inproceedings{kim2021selfreg,
  title={Selfreg: Self-supervised contrastive regularization for domain generalization},
  author={Kim, Daehee and Yoo, Youngjun and Park, Seunghyun and Kim, Jinkyu and Lee, Jaekoo},
  booktitle={Proceedings of the IEEE/CVF international conference on computer vision},
  pages={9619--9628},
  year={2021}
}

@inproceedings{nam2021reducing,
  title={Reducing domain gap by reducing style bias},
  author={Nam, Hyeonseob and Lee, HyunJae and Park, Jongchan and Yoon, Wonjun and Yoo, Donggeun},
  booktitle={Proceedings of the IEEE/CVF conference on computer vision and pattern recognition},
  pages={8690--8699},
  year={2021}
}

@article{shi2021gradient,
  title={Gradient matching for domain generalization},
  author={Shi, Yuge and Seely, Jeffrey and Torr, Philip HS and Siddharth, Narayanaswamy and Hannun, Awni and Usunier, Nicolas and Synnaeve, Gabriel},
  journal={arXiv preprint arXiv:2104.09937},
  year={2021}
}

@inproceedings{li2018learning,
  title={Learning to generalize: Meta-learning for domain generalization},
  author={Li, Da and Yang, Yongxin and Song, Yi-Zhe and Hospedales, Timothy},
  booktitle={Proceedings of the AAAI conference on artificial intelligence},
  volume={32},
  number={1},
  year={2018}
}

@article{blanchard2021domain,
  title={Domain generalization by marginal transfer learning},
  author={Blanchard, Gilles and Deshmukh, Aniket Anand and Dogan, Urun and Lee, Gyemin and Scott, Clayton},
  journal={Journal of machine learning research},
  volume={22},
  number={2},
  pages={1--55},
  year={2021}
}

@inproceedings{huang2020self,
  title={Self-challenging improves cross-domain generalization},
  author={Huang, Zeyi and Wang, Haohan and Xing, Eric P and Huang, Dong},
  booktitle={European conference on computer vision},
  pages={124--140},
  year={2020},
  organization={Springer}
}

@inproceedings{krueger2021out,
  title={Out-of-distribution generalization via risk extrapolation (rex)},
  author={Krueger, David and Caballero, Ethan and Jacobsen, Joern-Henrik and Zhang, Amy and Binas, Jonathan and Zhang, Dinghuai and Le Priol, Remi and Courville, Aaron},
  booktitle={International conference on machine learning},
  pages={5815--5826},
  year={2021},
  organization={PMLR}
}

@article{zhang2021adaptive,
  title={Adaptive risk minimization: Learning to adapt to domain shift},
  author={Zhang, Marvin and Marklund, Henrik and Dhawan, Nikita and Gupta, Abhishek and Levine, Sergey and Finn, Chelsea},
  journal={Advances in neural information processing systems},
  volume={34},
  pages={23664--23678},
  year={2021}
}

@article{sagawa2019distributionally,
  title={Distributionally robust neural networks for group shifts: On the importance of regularization for worst-case generalization},
  author={Sagawa, Shiori and Koh, Pang Wei and Hashimoto, Tatsunori B and Liang, Percy},
  journal={arXiv preprint arXiv:1911.08731},
  year={2019}
}

@article{zhou2021domain,
  title={Domain generalization with mixstyle},
  author={Zhou, Kaiyang and Yang, Yongxin and Qiao, Yu and Xiang, Tao},
  journal={arXiv preprint arXiv:2104.02008},
  year={2021}
}

@inproceedings{li2018domainMMD,
  title={Domain generalization with adversarial feature learning},
  author={Li, Haoliang and Pan, Sinno Jialin and Wang, Shiqi and Kot, Alex C},
  booktitle={Proceedings of the IEEE conference on computer vision and pattern recognition},
  pages={5400--5409},
  year={2018}
}

@article{liu2021isometric,
  title={Isometric propagation network for generalized zero-shot learning},
  author={Liu, Lu and Zhou, Tianyi and Long, Guodong and Jiang, Jing and Dong, Xuanyi and Zhang, Chengqi},
  journal={arXiv preprint arXiv:2102.02038},
  year={2021}
}

@inproceedings{cendra2024promptccd,
  title={Promptccd: Learning gaussian mixture prompt pool for continual category discovery},
  author={Cendra, Fernando Julio and Zhao, Bingchen and Han, Kai},
  booktitle={European conference on computer vision},
  pages={188--205},
  year={2024},
  organization={Springer}
}

@article{zhang2022grow,
  title={Grow and merge: A unified framework for continuous categories discovery},
  author={Zhang, Xinwei and Jiang, Jianwen and Feng, Yutong and Wu, Zhi-Fan and Zhao, Xibin and Wan, Hai and Tang, Mingqian and Jin, Rong and Gao, Yue},
  journal={Advances in Neural Information Processing Systems},
  volume={35},
  pages={27455--27468},
  year={2022}
}

@inproceedings{kim2023proxy,
  title={Proxy anchor-based unsupervised learning for continuous generalized category discovery},
  author={Kim, Hyungmin and Suh, Sungho and Kim, Daehwan and Jeong, Daun and Cho, Hansang and Kim, Junmo},
  booktitle={Proceedings of the IEEE/CVF international conference on computer vision},
  pages={16688--16697},
  year={2023}
}

@inproceedings{wu2023metagcd,
  title={Metagcd: Learning to continually learn in generalized category discovery},
  author={Wu, Yanan and Chi, Zhixiang and Wang, Yang and Feng, Songhe},
  booktitle={Proceedings of the IEEE/CVF International Conference on Computer Vision},
  pages={1655--1665},
  year={2023}
}

@article{oquab2023DINOv2,
  title={Dinov2: Learning robust visual features without supervision},
  author={Oquab, Maxime and Darcet, Timoth{\'e}e and Moutakanni, Th{\'e}o and Vo, Huy and Szafraniec, Marc and Khalidov, Vasil and Fernandez, Pierre and Haziza, Daniel and Massa, Francisco and El-Nouby, Alaaeldin and others},
  journal={arXiv preprint arXiv:2304.07193},
  year={2023}
}

@inproceedings{he2022masked,
  title={Masked autoencoders are scalable vision learners},
  author={He, Kaiming and Chen, Xinlei and Xie, Saining and Li, Yanghao and Doll{\'a}r, Piotr and Girshick, Ross},
  booktitle={Proceedings of the IEEE/CVF conference on computer vision and pattern recognition},
  pages={16000--16009},
  year={2022}
}

@article{zamanitajeddin2024benchmarking,
  title={Benchmarking domain generalization algorithms in computational pathology},
  author={Zamanitajeddin, Neda and Jahanifar, Mostafa and Xu, Kesi and Siraj, Fouzia and Rajpoot, Nasir},
  journal={arXiv preprint arXiv:2409.17063},
  year={2024}
}

@article{kather2016multi,
  title={Multi-class texture analysis in colorectal cancer histology},
  author={Kather, Jakob Nikolas and Weis, Cleo-Aron and Bianconi, Francesco and Melchers, Susanne M and Schad, Lothar R and Gaiser, Timo and Marx, Alexander and Z{\"o}llner, Frank Gerrit},
  journal={Scientific reports},
  volume={6},
  number={1},
  pages={27988},
  year={2016},
  publisher={Nature Publishing Group UK London}
}

@article{kather2019predicting,
  title={Predicting survival from colorectal cancer histology slides using deep learning: A retrospective multicenter study},
  author={Kather, Jakob Nikolas and Krisam, Johannes and Charoentong, Pornpimol and Luedde, Tom and Herpel, Esther and Weis, Cleo-Aron and Gaiser, Timo and Marx, Alexander and Valous, Nektarios A and Ferber, Dyke and others},
  journal={PLoS medicine},
  volume={16},
  number={1},
  pages={e1002730},
  year={2019},
  publisher={Public Library of Science San Francisco, CA USA}
}

@article{javed2020cellular,
  title={Cellular community detection for tissue phenotyping in colorectal cancer histology images},
  author={Javed, Sajid and Mahmood, Arif and Fraz, Muhammad Moazam and Koohbanani, Navid Alemi and Benes, Ksenija and Tsang, Yee-Wah and Hewitt, Katherine and Epstein, David and Snead, David and Rajpoot, Nasir},
  journal={Medical image analysis},
  volume={63},
  pages={101696},
  year={2020},
  publisher={Elsevier}
}

@inproceedings{radford2021learning,
  title={Learning transferable visual models from natural language supervision},
  author={Radford, Alec and Kim, Jong Wook and Hallacy, Chris and Ramesh, Aditya and Goh, Gabriel and Agarwal, Sandhini and Sastry, Girish and Askell, Amanda and Mishkin, Pamela and Clark, Jack and others},
  booktitle={International conference on machine learning},
  pages={8748--8763},
  year={2021},
  organization={PmLR}
}

@article{krizhevsky2009learning,
  title={Learning multiple layers of features from tiny images},
  author={Krizhevsky, Alex and Hinton, Geoffrey and others},
  year={2009},
  publisher={Toronto, ON, Canada}
}

@article{
  lew2023gestur,
  title={Gradient estimation for unseen domain risk minimization with pre-trained models},
  author={Lew, Byounggyu and Son, Donghyun and Chang, Buru},
  booktitle={Proceedings of the IEEE/CVF International Conference on Computer Vision},
  pages={4436--4446},
  year={2023}
}

@article{arpit2022ensemble,
  title={Ensemble of averages: Improving model selection and boosting performance in domain generalization},
  author={Arpit, Devansh and Wang, Huan and Zhou, Yingbo and Xiong, Caiming},
  journal={Advances in Neural Information Processing Systems},
  volume={35},
  pages={8265--8277},
  year={2022}
}

@article{niu2022domain,
  title={Domain-unified prompt representations for source-free domain generalization},
  author={Niu, Hongjing and Li, Hanting and Zhao, Feng and Li, Bin},
  journal={arXiv preprint arXiv:2209.14926},
  year={2022}
}

@article{zhou2022learning,
  title={Learning to prompt for vision-language models},
  author={Zhou, Kaiyang and Yang, Jingkang and Loy, Chen Change and Liu, Ziwei},
  journal={International journal of computer vision},
  volume={130},
  number={9},
  pages={2337--2348},
  year={2022},
  publisher={Springer}
}

@article{zhang2023domain,
  title={Domain prompt learning for efficiently adapting clip to unseen domains},
  author={Zhang, Xin and Gu, Shixiang Shane and Matsuo, Yutaka and Iwasawa, Yusuke},
  journal={Transactions of the Japanese Society for Artificial Intelligence},
  volume={38},
  number={6},
  pages={B--MC2\_1},
  year={2023},
  publisher={The Japanese Society for Artificial Intelligence}
}

@inproceedings{shu2023clipood,
  title={Clipood: Generalizing clip to out-of-distributions},
  author={Shu, Yang and Guo, Xingzhuo and Wu, Jialong and Wang, Ximei and Wang, Jianmin and Long, Mingsheng},
  booktitle={International conference on machine learning},
  pages={31716--31731},
  year={2023},
  organization={PMLR}
}

@inproceedings{cho2023promptstyler,
  title={Promptstyler: Prompt-driven style generation for source-free domain generalization},
  author={Cho, Junhyeong and Nam, Gilhyun and Kim, Sungyeon and Yang, Hunmin and Kwak, Suha},
  booktitle={Proceedings of the IEEE/CVF International Conference on Computer Vision},
  pages={15702--15712},
  year={2023}
}

@inproceedings{lee2025domain,
  title={Domain generalization using large pretrained models with mixture-of-adapters},
  author={Lee, Gyuseong and Jang, Wooseok and Kim, Jinhyeon and Jung, Jaewoo and Kim, Seungryong},
  booktitle={2025 IEEE/CVF Winter Conference on Applications of Computer Vision (WACV)},
  pages={8259--8269},
  year={2025},
  organization={IEEE}
}

@inproceedings{cheng2024disentangled,
  title={Disentangled prompt representation for domain generalization},
  author={Cheng, De and Xu, Zhipeng and Jiang, Xinyang and Wang, Nannan and Li, Dongsheng and Gao, Xinbo},
  booktitle={Proceedings of the IEEE/CVF Conference on Computer Vision and Pattern Recognition},
  pages={23595--23604},
  year={2024}
}

@article{yu2024clipceil,
  title={Clipceil: Domain generalization through clip via channel refinement and image-text alignment},
  author={Yu, Xi and Yoo, Shinjae and Lin, Yuewei},
  journal={Advances in Neural Information Processing Systems},
  volume={37},
  pages={4267--4294},
  year={2024}
}

@article{gong2025causal,
  title={A Causal Intervention Method for Domain Generalization with a Self-Supervised Auxiliary Task},
  author={Gong, Qinkang and Pan, Yan and Lai, Hanjiang and Yin, Jian},
  journal={International Journal of Computer Vision},
  volume={133},
  number={10},
  pages={7110--7127},
  year={2025},
  publisher={Springer}
}

@inproceedings{liu2025caurdg,
  title={CauRDG: Enhancing Domain Generalization with Causal-Driven Semantic Consistency Reasoning},
  author={Liu, Zongxin and Liu, Yishu and Lu, Guangming and Luo, Xiaoling and Chen, Bingzhi},
  booktitle={Proceedings of the 33rd ACM International Conference on Multimedia},
  pages={11101--11110},
  year={2025}
}

@inproceedings{zhang2024mix,
  title={Mix background and foreground separately: Transformer-based Augmentation Strategies for Domain Generalization},
  author={Zhang, Zhongqiang and Cai, Fuhan and Liu, Duo and Liu, Ge and Fang, Xiangzhong},
  booktitle={2024 IEEE International Conference on Multimedia and Expo (ICME)},
  pages={1--6},
  year={2024},
  organization={IEEE}
}

@article{wang2025exploring,
  title={Exploring invariance matters for domain generalization},
  author={Wang, Shanshan and He, Houmeng and Yang, Xun and Liu, Zhipu and Zhong, Yuanhong and Zhang, Xingyi and Wang, Meng},
  journal={IEEE Transactions on Image Processing},
  year={2025},
  publisher={IEEE}
}

@article{tan2025exploiting,
  title={Exploiting Layer Normalization Fine-tuning in Visual Transformer Foundation Models for Classification},
  author={Tan, Zhaorui and Pan, Tan and Huang, Kaizhu and Yu, Weimiao and Yao, Kai and Jiang, Chen and Wang, Qiufeng and Nguyen, Anh and Guo, Xin and Cheng, Yuan and others},
  journal={arXiv preprint arXiv:2508.07577},
  year={2025}
}

@inproceedings{de2023effectiveness,
  title={On the effectiveness of layernorm tuning for continual learning in vision transformers},
  author={De Min, Thomas and Mancini, Massimiliano and Alahari, Karteek and Alameda-Pineda, Xavier and Ricci, Elisa},
  booktitle={2023 IEEE/CVF International Conference on Computer Vision Workshops (ICCVW)},
  pages={3577--3586},
  year={2023},
  organization={IEEE}
}

@article{wang2024causal,
  title={Causal image modeling for efficient visual understanding},
  author={Wang, Feng and Yang, Timing and Yu, Yaodong and Ren, Sucheng and Wei, Guoyizhe and Wang, Angtian and Shao, Wei and Zhou, Yuyin and Yuille, Alan and Xie, Cihang},
  year={2024}
}

@article{rao2021dynamicvit,
  title={Dynamicvit: Efficient vision transformers with dynamic token sparsification},
  author={Rao, Yongming and Zhao, Wenliang and Liu, Benlin and Lu, Jiwen and Zhou, Jie and Hsieh, Cho-Jui},
  journal={Advances in neural information processing systems},
  volume={34},
  pages={13937--13949},
  year={2021}
}

@inproceedings{abnar2020quantifying,
  title={Quantifying attention flow in transformers},
  author={Abnar, Samira and Zuidema, Willem},
  booktitle={Proceedings of the 58th annual meeting of the association for computational linguistics},
  pages={4190--4197},
  year={2020}
}

@article{carter2019sis,
  title={What made you do this? Understanding black-box decisions with sufficient input subsets},
  author={Brandon Carter and Jonas W. Mueller and Siddhartha Jain and David Kenneth Gifford},
  journal={ArXiv},
  year={2018},
  volume={abs/1810.03805},
  url={https://api.semanticscholar.org/CorpusID:52946369}
}

@inproceedings{carter2021overinterpretation,
  title={Overinterpretation reveals image classification model pathologies},
  author={Brandon Carter and Siddhartha Jain and Jonas W. Mueller and David Gifford},
  booktitle={Neural Information Processing Systems},
  year={2020},
  url={https://api.semanticscholar.org/CorpusID:213002712}
}

@article{wang2021caam,
  title={Causal Attention for Unbiased Visual Recognition},
  author={Tan Wang and Chan Zhou and Qianru Sun and Hanwang Zhang},
  journal={2021 IEEE/CVF International Conference on Computer Vision (ICCV)},
  year={2021},
  pages={3071-3080},
  url={https://api.semanticscholar.org/CorpusID:237213244}
}

@article{shi2026registers,
  title={Vision Transformers Need More Than Registers},
  author={Cheng Shi and Yizhou Yu and Si-Bei Yang},
  journal={ArXiv},
  year={2026},
  volume={abs/2602.22394},
  url={https://api.semanticscholar.org/CorpusID:286083621}
}

@article{wang2025hilo,
  title={HiLo: A Learning Framework for Generalized Category Discovery Robust to Domain Shifts},
  author={Hongjun Wang and Sagar Vaze and Kai Han},
  journal={ArXiv},
  year={2024},
  volume={abs/2408.04591},
  url={https://api.semanticscholar.org/CorpusID:271768898}
}

@article{rathore2025dg2cd,
  title={When Domain Generalization meets Generalized Category Discovery: An Adaptive Task-Arithmetic Driven Approach},
  author={Vaibhav Varish Singh Rathore and Shubhranil B and Saikat Dutta and Sarthak Mehrotra and Zsolt Kira and Biplab Banerjee},
  journal={2025 IEEE/CVF Conference on Computer Vision and Pattern Recognition (CVPR)},
  year={2025},
  pages={4905-4915},
  url={https://api.semanticscholar.org/CorpusID:277113473}
}

@article{tan2025saving,
  title={Saving for the future: Enhancing generalization via partial logic regularization},
  author={Tan, Zhaorui and Hu, Yijie and Yang, Xi and Wang, Qiufeng and Nguyen, Anh and Huang, Kaizhu},
  journal={arXiv preprint arXiv:2508.15317},
  year={2025}
}

@ARTICLE{lin2025constrained,
  author={Lin, Jianxin and Tang, Yongqiang and Wang, Junping and Zhang, Wensheng},
  journal={IEEE Trans. Neural Netw. Learn. Syst.}, 
  title={Constrained Maximum Cross-Domain Likelihood for Domain Generalization}, 
  year={2025},
  volume={36},
  number={2},
  pages={2013--2027},
  doi={10.1109/TNNLS.2023.3292242}
}

@ARTICLE{duan2025hyperbolic,
  author={Duan, Yu and Nie, Feiping and Chen, Huimin and Hu, Zhanxuan and Wang, Rong and Li, Xuelong},
  journal={IEEE Trans. Neural Netw. Learn. Syst.}, 
  title={Hyperbolic Hierarchical Representation Learning for Generalized Category Discovery}, 
  year={2025},
  volume={36},
  number={12},
  pages={20050--20064},
  doi={10.1109/TNNLS.2025.3597074}
}
}

\clearpage
\appendix
\section{Detailed Related Work}
\label{app:detailed_related_work}

\textbf{Generalization across shifts.}
Multi-domain generalization (mDG) learns from multiple source domains to generalize to unseen domains without target-domain training data~\cite{lin2025constrained}.
Whereas DANN uses unlabeled target data for adaptation~\cite{ganin2016domain}, multi-source adversarial variants align source distributions.
CDANN, CIDG, and MDA learn class-conditional invariance~\cite{li2018deep,li2018domain,hu2020domain}; MIRO and GMDG regularize pretrained representations~\cite{cha2022miro,tan2024rethinking}.
Continual category discovery (CCD) combines GCD and continual learning to discover categories sequentially without a predefined total count while limiting forgetting~\cite{zhang2022grow,kim2023proxy,wu2023metagcd,cendra2024promptccd}.
PromptCCD's dynamic prompts decouple category expansion from fixed classifiers and support knowledge retention.
L-Reg uses logic-based regularization across shifts~\cite{tan2024interpret}; we adopt its mDG+GCD protocol, training on partially labeled source domains for known-class recognition and unknown-class clustering on unseen domains.
HiLo uses shifted-domain unlabeled data containing known and novel classes~\cite{wang2025hilo}, whereas DG$^2$CD-Net studies unseen-target GCD with source-only training~\cite{rathore2025dg2cd}.
We also evaluate PatchGen with HiLo.

\textbf{Causal generalization.}
CauRDG uses prototype-guided disentanglement~\cite{liu2025caurdg}, SMIDG models latent confounders with interventions~\cite{wang2025exploring}, and CI-DGA aligns causal structures across domains~\cite{gong2025causal}, primarily addressing domain-level effects.
Foreground--background mixing reduces contextual correlations~\cite{zhang2024mix} and is most natural with clear object--context boundaries, rarer in histopathology.
Causal image modeling and explanation study how visual factors affect model behavior~\cite{wang2024causal}.
Predictive sufficiency complements causal-variable or effect identification by guiding decision-relevant patch selection.

\textbf{Predictive subsets and patch selection.}
DynamicViT prunes tokens for efficiency~\cite{rao2021dynamicvit}, but attention visualization supports post-hoc interpretation~\cite{abnar2020quantifying}.
Sufficient input subsets identify minimal observed features preserving a fixed predictor's decision with remaining inputs masked~\cite{carter2019sis}; overinterpretation reveals confident predictions from visually nonsalient pixels~\cite{carter2021overinterpretation}.
These decision-based explanations differ from our conditional-sufficiency oracle.
CaaM addresses contextual confounding with intervention-inspired attention~\cite{wang2021caam}, while Shi et al.\ modify CLS-token integration to address background shortcuts~\cite{shi2026registers}.
PatchGen jointly trains sample-dependent soft masks and selected/complementary aggregates with task and generalization objectives, learning a predictive-subset proxy to improve pretrained representations.

\section{Details of Proofs}
\label{app:proofs}
\setcounter{equation}{0}
\renewcommand{\theequation}{\thesection.\arabic{equation}}
\renewcommand{\theHequation}{\thesection.\arabic{equation}}

\begin{proof}[\normalfont\textbf{Proof of Prop.~\ref{prop:oracle_generalization}}]
Stable predictive sufficiency gives $\mathbb P(Y\mid\Phi_{\mathrm{all}})=\mathbb P(Y\mid\Phi_C)$.
Hence the conditional label law, and therefore the Bayes decision, depends
only on $\Phi_C$.
The unrestricted Bayes risks based on $\Phi_C$ and $\Phi_{\mathrm{all}}$ are equal.

For the complexity bound, let $u_i=\Phi_{\mathbf m_C^\star(X_i)}(X_i)$. Since at most $s$ patches are selected and each has norm at most $B$, $\|u_i\|_2\le B\sqrt s$.
Therefore,
\begin{align}
&\widehat{\mathfrak R}_n(\mathcal H_{\mathbf m_C^\star})
=\mathbb E_\sigma\Bigl[
\sup\nolimits_{\|\mathbf v\|_2\le\Lambda}(1/n)
\sum\nolimits_{i=1}^n\sigma_i\langle\mathbf v,u_i\rangle
\Bigr]\nonumber\\[-1pt]
&\quad=\tfrac{\Lambda}{n}\mathbb E_\sigma
\Bigl\|\sum\nolimits_{i=1}^n\sigma_i u_i\Bigr\|_2
\le\tfrac{\Lambda}{n}\sqrt{\mathbb E_\sigma
\Bigl\|\sum\nolimits_{i=1}^n\sigma_i u_i\Bigr\|_2^2}
\nonumber\\[-1pt]
&\quad=(\Lambda/n)\sqrt{\sum\nolimits_{i=1}^n\|u_i\|_2^2}
\le\Lambda B\sqrt{s}/\sqrt n.
\end{align}
The all-patch bound follows from $\|u_i\|_2\le B\sqrt P$.
\end{proof}

\begin{corollary}[Oracle-subset generalization bound]
\label{cor:oracle_gen_bound}
Under the conditions of Prop.~\ref{prop:oracle_generalization}, fix the backbone and oracle mask function independently of an i.i.d. sample of size $n$ from the distribution defining $R$. If $\ell\in[0,1]$ is $\rho$-Lipschitz in the score, then, with probability at least $1-\delta$, every $h\in\mathcal H_C$ satisfies
\[
R(h)\le\widehat R(h)+2\rho\Lambda B\sqrt{s/n}
+3\sqrt{\log(2/\delta)/(2n)}.
\]
The all-patch bound replaces $s$ by $P$.
\end{corollary}

\begin{proof}[\normalfont\textbf{Proof of Corollary~\ref{cor:oracle_gen_bound}}]
The result follows by applying the contraction inequality to the $\rho$-Lipschitz loss and then the standard Rademacher generalization bound using the complexity estimate established in Prop.~\ref{prop:oracle_generalization}.
\end{proof}

\begin{proof}[\normalfont\textbf{Proof of Lemma~\ref{lem:mask_stability}}]
By non-expansiveness of $G$,
\begin{align}
&
\left\|
G\!\left(\Phi_{\mathbf m}(X)\right)
-
G\!\left(\Phi_{\mathbf m'}(X)\right)
\right\|_2
\!\le
\left\|
\Phi_{\mathbf m}(X)-\Phi_{\mathbf m'}(X)
\right\|_2. \nonumber
\end{align}
Moreover,
\begin{align}
\|\Phi_{\mathbf m}(X)-\Phi_{\mathbf m'}(X)\|_2^2
&=\sum\nolimits_{p=1}^P(m_p-m_p')^2\|z_p(X)\|_2^2
\nonumber\\
&\le B^2\|\mathbf m-\mathbf m'\|_2^2.
\end{align}
Taking square roots and using
$\|\mathbf a\|_2\le\|\mathbf a\|_1$ proves the claim.
\end{proof}

\begin{proof}[\normalfont\textbf{Proof of Lemma~\ref{lem:oracle_risk_deviation}}]
By Lemma~\ref{lem:mask_stability}, the distance
$\left\|\widetilde z_\phi^+(X)-\widetilde z_C(X)\right\|_2$
is at most $B\left\|\mathbf m_\phi(X)-\mathbf m_C^\star(X)\right\|_1$.
The Lipschitz condition therefore gives
\begin{align}
|R_d(h,\phi)-R_d(h,C)|
&\le
LB\,
\mathbb E_d
\left[
\left\|
\mathbf m_\phi(X)-\mathbf m_C^\star(X)
\right\|_1
\right]
\nonumber\\
&=
LBP\,\epsilon_d(\phi).
\end{align}
This completes the proof.
\end{proof}

\begin{proof}[\normalfont\textbf{Proof of Prop.~\ref{prop:data_shift}}]
Because $\widetilde z_C$ is a deterministic function of $\Phi_C$ and the same fixed refinement map is used in both domains, $R_d(h,C)=R_{d'}(h,C)$.
Using the triangle inequality and
Lemma~\ref{lem:oracle_risk_deviation},
\begin{align}
&|R_d(h,\phi)-R_{d'}(h,\phi)|
\le|R_d(h,\phi)-R_d(h,C)|\nonumber\\[-1pt]
&\quad+|R_d(h,C)-R_{d'}(h,C)|
+|R_{d'}(h,C)-R_{d'}(h,\phi)|\nonumber\\[-1pt]
&\quad\le LBP\bigl(\epsilon_d(\phi)+\epsilon_{d'}(\phi)\bigr).
\end{align}
This completes the proof.
\end{proof}

\begin{proof}[\normalfont\textbf{Proof of Prop.~\ref{prop:target_shift}}]
Consider any two samples $X$ and $X'$ with labels
$Y=c$ and $Y'=c'$, where $c\ne c'$.
For a sample $X$ with $Y=c$,
\begin{align}
    \|\widetilde z_\phi^+(X)-\mu_c\|_2
&\le \|\widetilde z_\phi^+(X)-\widetilde z_C(X)\|_2
\nonumber\\
&\quad+
\|\widetilde z_C(X)-\mu_c\|_2
\le \eta+r.
\end{align}
Likewise,
$\|\widetilde z_\phi^+(X')-\mu_{c'}\|_2\le\eta+r$.
The reverse triangle inequality gives
\begin{align}
&\|\widetilde z_\phi^+(X)-\widetilde z_\phi^+(X')\|_2
\ge\|\mu_c-\mu_{c'}\|_2
-\|\widetilde z_\phi^+(X)-\mu_c\|_2\nonumber\\[-1pt]
&\quad-\|\widetilde z_\phi^+(X')-\mu_{c'}\|_2
\ge\gamma-2(r+\eta).
\end{align}
This completes the proof.
\end{proof}

\begin{proof}[\normalfont\textbf{Proof of Corollary~\ref{cor:mask_target_separation}}]
The result follows immediately from
Lemma~\ref{lem:mask_stability} and Prop.~\ref{prop:target_shift}.
\end{proof}

\begin{proof}[\normalfont\textbf{Proof of Corollary~\ref{cor:all_shift}}]
The risk-discrepancy conclusion follows by applying Prop.~\ref{prop:data_shift} to the modeled data-shift component,
while the class-separation conclusion follows from
Prop.~\ref{prop:target_shift} for the target-shift component.
\end{proof}

\section{Method Details}
\label{app:method-details}

\textbf{Equivalent detached-reference form.}
The objective can be written with a detached complementary-confidence reference:
$\mathcal L_{neg} =\tfrac1{|\bar{\mathcal B}|}\sum\nolimits_{i\in\bar{\mathcal B}} \Bigl[\operatorname{sg}\!\bigl(p_h(\bar Y_i\mid\widetilde z_{\phi,i}^{-})\bigr) -p_h(\bar Y_i\mid\widetilde z_{\phi,i}^{+})\Bigr].$
Because the complementary term is detached,
$\nabla\mathcal L_{neg} =\nabla\Bigl[-(1/|\bar{\mathcal B}|) \sum\nolimits_{i\in\bar{\mathcal B}} p_h(\bar Y_i\mid\widetilde z_{\phi,i}^{+})\Bigr].$
Thus, this form has the same trainable effect as $\mathcal L_{\mathrm{conf}}$: its gradient updates the selected representation while treating complementary confidence as a fixed reference.

\section{Experimental details}
\label{app:exp-results}

\textbf{More experimental details of GCD.}
All method use DINO-pretrained ViT-B/16. We use 100 known and
100 novel classes, with labels available for 50\% of
the training images from known classes. Models are
trained for 200 epochs with batch size 128 and seed 0.
We retain each baseline's training objectives and add
PatchGen through gated residual feature fusion, without
additional PatchGen regularization losses. We report
final-epoch clustering accuracy on the unlabeled training
pool using Hungarian matching.

\textbf{More experimental details of GCD+DS.}
From official HiLo DINO ViT-B/16 checkpoints, we freeze the backbone and original teacher head, training PatchGen, a zero-initialized residual projection, and the full student classifier for five epochs (seed $0$). The classifier takes CLS plus the projected selected feature. The loss combines HiLo's semantic GCD loss, teacher KL (weight $1$, temperature $0.1$), and PatchGen regularizers ($\lambda_{ms}=\lambda_{sim}=0.001$; confidence weight $5$). Confidence uses logits divided by $0.1$; similarity uses pre-projection selector features.\footnote{Unlabeled targets are detached student argmax predictions; $\mathcal L_{sim}$ is additionally normalized by the minibatch active-class count.} Both datasets share this setting from an exploratory sweep. We report final-epoch accuracy with pooled Hungarian matching over $45$ conditions ($9$ corruption types $\times$ $5$ severities).

\section{Detailed Results}
\label{app:detailed-results}
\makeatletter\setlength{\@dblfptop}{0pt}\setlength{\@dblfpsep}{10pt}\makeatother
\noindent\textbf{Sensitivity to loss weights and random seeds.}
\noindent Table~\ref{tab:analysis} reports the complete TerraIncognita mDG results underlying Fig.~\ref{fig:analysis}. The three blocks evaluate sensitivity to $\lambda_{sim}$, $\lambda_{ms}$, and the random seed, respectively. Each loss weight is varied over $\{10^{-4},10^{-3},10^{-2}\}$ while the other is fixed at $10^{-3}$ and the seed at 1; the seed comparison uses $\{0,1,2\}$ with both weights fixed at $10^{-3}$. Columns 100, 38, 43, and 46 report accuracy (\%) for each held-out location, and Avg. reports their mean. Each setting is evaluated without ($S=0$) and with ($S=1$) SWAD.\par

\begin{table*}[!t]
\centering
\setlength{\abovecaptionskip}{0pt}
\caption{TerraIncognita mDG sensitivity and robustness. $S$: SWAD off/on (0/1).}
\label{tab:analysis}
{\fontsize{8}{8.8}\selectfont
\setlength{\tabcolsep}{2pt}
\renewcommand{\arraystretch}{0.92}
\setlength{\aboverulesep}{0.4pt}
\setlength{\belowrulesep}{0.4pt}
\begin{minipage}[t]{0.32\linewidth}
\vspace{0pt}\centering
\begin{tabular*}{\linewidth}{@{\extracolsep{\fill}}ccrrrrr@{}}
\toprule
Value & $S$ & \multicolumn{4}{c}{Location} & Avg. \\
& & 100 & 38 & 43 & 46 & \\
\midrule
\multicolumn{7}{l}{$\lambda_{sim}$: $\lambda_{ms}=$\texttt{1e-3}, seed 1} \\
\texttt{1e-2} & 0 & 75.8 & 63.4 & 71.0 & 62.0 & 68.1 \\
\rowcolor{mdgshade}\texttt{1e-2} & 1 & 76.6 & 64.7 & 72.7 & 64.5 & 69.6 \\
\texttt{1e-3} & 0 & 76.1 & 61.8 & 68.8 & 60.1 & 66.7 \\
\rowcolor{mdgshade}\texttt{1e-3} & 1 & 79.1 & 63.8 & 71.6 & 66.3 & 70.2 \\
\texttt{1e-4} & 0 & 76.2 & 61.2 & 70.2 & 64.8 & 68.1 \\
\rowcolor{mdgshade}\texttt{1e-4} & 1 & 77.2 & 64.0 & 71.1 & 65.9 & 69.6 \\
\bottomrule
\end{tabular*}\end{minipage}
\hfill
\begin{minipage}[t]{0.32\linewidth}
\vspace{0pt}\centering
\begin{tabular*}{\linewidth}{@{\extracolsep{\fill}}ccrrrrr@{}}
\toprule
Value & $S$ & \multicolumn{4}{c}{Location} & Avg. \\
& & 100 & 38 & 43 & 46 & \\
\midrule
\multicolumn{7}{l}{$\lambda_{ms}$: $\lambda_{sim}=$\texttt{1e-3}, seed 1} \\
\texttt{1e-2} & 0 & 74.5 & 61.1 & 70.4 & 63.3 & 67.3 \\
\rowcolor{mdgshade}\texttt{1e-2} & 1 & 78.8 & 59.5 & 72.3 & 65.0 & 68.9 \\
\texttt{1e-3} & 0 & 76.1 & 61.8 & 68.8 & 60.1 & 66.7 \\
\rowcolor{mdgshade}\texttt{1e-3} & 1 & 79.1 & 63.8 & 71.6 & 66.3 & 70.2 \\
\texttt{1e-4} & 0 & 78.2 & 61.9 & 69.6 & 62.8 & 68.1 \\
\rowcolor{mdgshade}\texttt{1e-4} & 1 & 77.2 & 62.5 & 69.8 & 65.0 & 68.6 \\
\bottomrule
\end{tabular*}\end{minipage}
\hfill
\begin{minipage}[t]{0.32\linewidth}
\vspace{0pt}\centering
\begin{tabular*}{\linewidth}{@{\extracolsep{\fill}}ccrrrrr@{}}
\toprule
Value & $S$ & \multicolumn{4}{c}{Location} & Avg. \\
& & 100 & 38 & 43 & 46 & \\
\midrule
\multicolumn{7}{l}{Seed: $\lambda_{ms}=\lambda_{sim}=$\texttt{1e-3}} \\

0 & 0 & 75.7 & 61.5 & 71.3 & 60.0 & 67.1 \\
\rowcolor{mdgshade}0 & 1 & 79.3 & 64.0 & 72.1 & 64.4 & 70.0 \\
1 & 0 & 76.1 & 61.8 & 68.8 & 60.1 & 66.7 \\
\rowcolor{mdgshade}1 & 1 & 79.1 & 63.8 & 71.6 & 66.3 & 70.2 \\
2 & 0 & 76.0 & 63.2 & 68.6 & 58.8 & 66.6 \\
\rowcolor{mdgshade}2 & 1 & 77.7 & 65.0 & 70.2 & 66.1 & 69.8 \\
\bottomrule
\end{tabular*}\end{minipage}}
\end{table*}

\noindent\textbf{PatchGen with full-model fine-tune.}
\noindent
Table~\ref{tab:mDG_results_full_finetune} reports full-model fine-tuning across all mDG datasets. It may disturb pretrained knowledge, yielding lower performance than LN tuning. With full-model tuning and SWAD, PatchGen improves the corresponding baseline average by $0.3$ percentage points, with nonuniform gains across datasets.

\begin{table*}[!t]
\begingroup

\fontsize{7.5}{8.3}\selectfont
\setlength{\tabcolsep}{1.3pt}\renewcommand{\arraystretch}{0.94}
\setlength{\abovecaptionskip}{3pt}
\setlength{\aboverulesep}{0.22ex}\setlength{\belowrulesep}{0.22ex}

\centering
\caption{Additional mDG results comparing LN and full finetuning.}
\label{tab:mDG_results_full_finetune}
{
\begin{tabular*}{\linewidth}{@{}l@{\extracolsep{\fill}}ccccc|c@{}}
 \toprule
  Test domain & \bf P & \bf V & \bf O & \bf T & \bf D & \bf Avg. \\ \midrule
Method (Pretrain: DINOv2 ViT-B/14) & \multicolumn{6}{c}{\bf Results of baselines and ours} \\
\cmidrule{2-7}
LP* &93.0&77.0&84.2&56.7&53.5&72.9\\
LP+LN finetune (Baseline)* & 96.2 &80.0 &86.1 &61.5 &56.3 &76.0 \\

Baseline + SWAD* & 96.7&80.3&86.8&64.7&58.3&77.4 \\
\midrule
LP+full finetune (Baseline)* & 95.9 &  82.6 & 82.2 & 61.6 & 56.2 & 75.7 \\
Baseline+SWAD* & 96.5 & 82.2 & 82.3 & 64.8 & 59.3 & 77.0\\
\rowcolor{mygray}\textbf{PatchGen}+full finetune+SWAD &96.6&82.7&85.3&62.3&59.5&77.3\\
\bottomrule 

\end{tabular*}%
}
\par\endgroup
\end{table*}

\noindent\textbf{Ablation of patch-selection components.}
\noindent Fig.~\ref{fig:ablation} visualizes the Avg. columns of the main-text ablation results in Table~\ref{tab:ablation}. Each horizontal-axis ID refers to the identically numbered table row, whose Method column defines the configuration. IDs 1.1--1.2 compare the LP+LN baseline (B) and B + GMDG; IDs 2.1--2.4 compare static patch masking, dimension-wise masking, standard attention, and attention with mask suppression. IDs 3.1--3.5 evaluate the sample-dependent selector and its training objectives, with 3.5 denoting the complete PatchGen model. Each point is the mean accuracy across PACS, VLCS, OfficeHome, TerraIncognita, and DomainNet, shown separately without and with SWAD.\par

\begin{figure*}[!ht]
\centering
\includegraphics[width=0.8\linewidth]{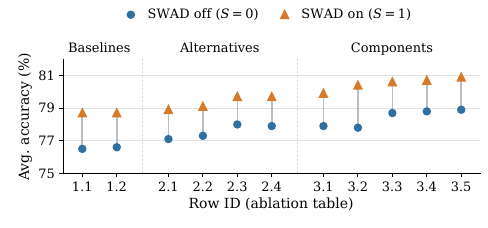}
\caption{\textbf{Ablation averages} across five mDG datasets, with and without SWAD. Horizontal-axis IDs match the rows of Table~\ref{tab:ablation}: 1.1--1.2 are baselines, 2.1--2.4 are alternative masking and attention designs, and 3.1--3.5 evaluate the proposed components.}
\label{fig:ablation}
\end{figure*}

\end{document}